\documentclass[10pt]{article} 
\usepackage[preprint]{tmlr}

\usepackage{hyperref}
\usepackage{url}

\title{Asymptotically and Minimax Optimal Regret Bounds for \\ Multi-Armed Bandits with Abstention}

\author{\name Junwen Yang \email junwen\_yang@u.nus.edu\\
       \addr Institute of Operations Research and Analytics\\
       National University of Singapore
       \TMLRAND
       \name Tianyuan Jin \email tianyuan@u.nus.edu \\
       \addr Department of Electrical and Computer Engineering \\
       National University of Singapore
       \TMLRAND
       \name Vincent Y. F. Tan \email vtan@nus.edu.sg \\
       \addr Department of Mathematics\\
        Department of Electrical and Computer Engineering\\
        Institute of Operations Research and Analytics\\
       National University of Singapore
       }
       
\def\month{03}  
\def\year{2026} 
\def\openreview{\url{https://openreview.net/forum?id=AYp5zOcFdA}} 

\usepackage{amsmath,amssymb,url,braket}
\usepackage{soul}

\usepackage{lipsum}
\usepackage{amsthm}
\usepackage{cleveref}
\usepackage{mathrsfs}
\usepackage{overpic}
\usepackage{tikz}
\usepackage{subfigure}

\allowdisplaybreaks[4]

\usepackage{adjustbox}
\usepackage{nicefrac}
\usepackage{comment}
\usepackage{subfigure}
\usepackage{mathtools}
\usepackage{etoolbox}
\AtBeginEnvironment{dcases}{\renewcommand\baselinestretch{0.5}\selectfont}
\usepackage{array, makecell}  
\usepackage{enumitem}
\usepackage{graphicx}
\usepackage{algorithm}
\usepackage{algorithmic}
\usepackage{booktabs}

\usepackage{xcolor}
\usepackage{soul}

\DeclareMathOperator{\E}{\mathbb{E}}

\DeclareMathOperator*{\argmax}{arg\,max}

\theoremstyle{plain}
\newtheorem{theorem}{Theorem}

\newtheorem{lemma}{Lemma}

\theoremstyle{definition}
\newtheorem{definition}{Definition}

\theoremstyle{remark}

\newtheorem{remark}{Remark}

\newcommand{\PP}{\mathbb{P}}

\newcommand{\CA}{\mathrm{CA}}
\newcommand{\FRG}{\mathrm{RG}}
\newcommand{\FRW}{\mathrm{RW}}

\begin{document}

\maketitle

\begin{abstract}
We introduce a novel extension of the canonical multi-armed bandit problem that incorporates an additional strategic innovation: \emph{abstention}. In this enhanced framework, the agent is not only tasked with selecting an arm at each time step, but also has the option to {\em abstain} from accepting the stochastic instantaneous reward before observing it. When opting for abstention, the agent either suffers a fixed regret or gains a guaranteed reward. This added layer of complexity naturally prompts the key question: can we develop algorithms that are both computationally efficient and asymptotically and minimax optimal in this setting? We answer this question in the affirmative by designing and analyzing algorithms whose regrets meet their corresponding information-theoretic lower bounds. Our results offer valuable quantitative insights into the benefits of the abstention option, laying the groundwork for further exploration in other online decision-making problems with such an option. Extensive numerical experiments validate our theoretical results, demonstrating that our approach not only advances theory but also has the potential to deliver significant practical benefits.
\end{abstract}

\section{Introduction}
\label{section_intro}
In online decision-making, the multi-armed bandit model,  originally introduced by \citet{thompson1933likelihood}, has long served as a quintessential benchmark for capturing the delicate interplay between exploration and exploitation. In stochastic multi-armed bandit problems, the agent sequentially selects an arm from the given set at each time step and subsequently observes a random reward associated with the chosen arm. To maximize cumulative rewards, the agent must strike a balance between the persistent pursuit of the arm with the highest estimated reward (exploitation) and the adventurous exploration of other arms to gain a deeper understanding of their potential (exploration). This fundamental challenge finds applications across a wide array of domains, ranging from optimizing advertising campaigns to fine-tuning recommendation systems.

However, real-world scenarios often come fraught with complexities that challenge the simplicity of the canonical bandit model. One notable complexity arises when the agent is equipped with an additional option to abstain from accepting the stochastic instantaneous reward before observing it. This added layer of decision-making enriches the strategic landscape, altering how the agent optimally navigates the trade-off between exploration and exploitation.

Consider, for example, the domain of clinical trials.\footnote{ {We highlight that this example is a stylized mathematical abstraction used solely to motivate the abstention framework. It does not capture the complex ethical and practical constraints of actual clinical trials and should not be interpreted as a proposal for real-world medical protocols.}} When evaluating potentially hazardous medical treatments, researchers can proactively deploy safeguards such as preemptive medications or consider purchasing specialized insurance packages to shield against possible negative consequences. However, these protective measures come with costs, which may be modeled as either fixed regrets or fixed rewards in the context of the clinical study's cumulative regret. In these scenarios,  researchers have the option to observe the outcomes of a treatment while abstaining from incurring the associated random regret through these costly prearranged measures. Subsequently, these observed outcomes can be leveraged to inform more efficacious future experimental designs.  Consequently, opting for abstention has the potential to promote more responsible decision-making and reduce the overall cumulative regret of the study. 

To further illustrate our motivation, consider another concrete example in the realm of online advertising, where companies frequently grapple with the challenge of allocating their advertising budget optimally across multiple platforms (Google Ads, LinkedIn, etc.) to promote their products. This scenario naturally parallels the framework of multi-armed bandits, with each arm representing a distinct advertising platform. Typically, these platforms operate on a \emph{cost-per-click} pricing model, where payment is made for each click generated. However, what companies truly care about is the outcome, such as conversions or purchases, which are random variables sampled from unknown distributions. When faced with such uncertainty and the possibility of underperformance, companies may instead opt for a \emph{cost-per-action} strategy, which involves a predetermined payment for a specific outcome (i.e., payment is made only when a desired outcome, such as a conversion or purchase, occurs). This strategy essentially serves as an abstention option, where instead of risking uncertain returns (from unknown distributions), companies choose to pay a fixed cost for a known outcome. By embracing the abstention option, companies can mitigate the uncertainties associated with stochastic outcomes and ensure a deterministic result. Moreover, by opting for abstention, the company still retains the ability to observe the advertising conversion rate of the platform, enabling them to gather valuable insights for future decision-making.

Building upon this challenge, we introduce an innovative extension to the canonical multi-armed bandit model that incorporates abstention as a legitimate strategic option. At each time step, the agent not only selects which arm to pull but also decides whether to abstain. Depending on how the abstention option impacts the cumulative regret, which is the agent's primary optimization objective, our abstention model offers two complementary settings, namely, the fixed-regret setting where abstention results in a constant regret, and the fixed-reward setting where abstention yields a deterministic reward. 
Collectively, these settings provide the agent with a comprehensive toolkit for adeptly navigating the complicated landscape of online decision-making.

\paragraph{Main contributions.} Our main results and contributions are summarized as follows:
\begin{enumerate}[label = (\roman*)] 
\item In Section~\ref{section_model}, we provide a rigorous  formulation of the multi-armed bandit model with abstention. Our focus is on cumulative regret minimization across two distinct yet complementary settings: \emph{fixed-regret} and \emph{fixed-reward}. These settings give rise to different metrics, each offering unique analytical insights. Importantly, both settings encompass the canonical bandit model as a particular case.

\item In the fixed-regret setting, we judiciously combine two   abstention criteria into a Thompson Sampling-based algorithm proposed by \citet{jin23b}. This integration ensures compatibility with the abstention option, as elaborated in Algorithm~\ref{algo1}.  The first abstention criterion employs a carefully constructed lower confidence bound, while the second is tailored to mitigate worst-case scenarios. We establish both asymptotic and minimax upper bounds on the cumulative regret, requiring application of analytical techniques inspired by both Thompson Sampling and UCB methodologies. Furthermore, we derive corresponding lower bounds, thereby showcasing the concurrent attainment of asymptotic and minimax optimality by our algorithm. Our approach entails extensive application of analytical techniques inspired by both Thompson Sampling and UCB methodologies.

\item In the fixed-reward setting, we introduce a general strategy, outlined in Algorithm~\ref{algo2}. This method transforms any algorithm that is both asymptotically and minimax optimal in the canonical model to one that also accommodates the abstention option. Remarkably, this strategy maintains its universal applicability and straightforward implementation while provably achieving both forms of optimality---asymptotic and minimax.

\item To   corroborate our theoretical contributions, we conduct a series of numerical experiments in Appendix~\ref{appendix_experiment}. These experiments substantiate the effectiveness of our algorithms and highlight the performance gains achieved through the inclusion of the abstention option. 

\end{enumerate}

\subsection{Related Work}
\paragraph{Canonical multi-armed bandits.} The study of cumulative regret minimization in canonical multi-armed bandits has attracted considerable scholarly focus. Within this domain, two dominant paradigms for evaluating optimality metrics  have emerged: asymptotic optimality and minimax optimality. In essence, the former considers the behavior of algorithms as the time horizon approaches infinity for a specific problem instance, while the latter seeks to minimize the worst-case regret over all possible instances.  A diverse array of policies have been rigorously established to achieve asymptotic optimality across various model settings. Notable examples include UCB2 \citep{auer2002finite}, DMED \citep{honda2010asymptotically}, KL-UCB \citep{cappe2013kullback}, and Thompson Sampling \citep{agrawal2012analysis, kaufmann2012thompson}. In the context of the worst-case regret, MOSS \citep{audibert2009minimax} stands out as the pioneering method that has been verified to be minimax optimal.
Remarkably, $\text{KL-UCB}^{++}$ \citep{menard2017minimax} became the first algorithm proved to achieve both asymptotic and minimax optimality. Recently, \citet{jin23b} introduced Less-Exploring Thompson Sampling, an innovation that boosts computational efficiency compared to classical Thompson Sampling while concurrently achieving asymptotic and minimax optimality. For a comprehensive survey of bandit algorithms, we refer to \citet{lattimore2020bandit}.

\paragraph{Machine learning with abstention.} Starting with the seminal work of \citet{chow1957optimum, chow1970optimum}, the concept of learning with {\em  abstention} (also referred to as {\em  rejection}) has been extensively explored in various machine learning paradigms. These include  classification \citep{herbei2006classification, bartlett2008classification, cortes2016learning}, ranking \citep{cheng2010predicting, mao2023ranking} and regression \citep{wiener2012pointwise, zaoui2020regression, kalai2021towards}.

Within this broad spectrum of research, our work is most directly related to those that explore the role of abstention in the context of {\em online learning}. To the best of our knowledge, \citet{cortes2018online} firstly incorporated the abstention option into the problem of online prediction with expert advice \citep{littlestone1994weighted}. In their model, at each time step, each expert has the option to either make a prediction based on the given input or abstain from doing so. When the agent follows the advice of an expert who chooses to abstain, the true label of the input remains undisclosed, and the learner incurs a known fixed loss.
Subsequently, \citet{neu2020fast} introduced a different abstention model, which is more similar to ours. Here, the abstention option is only available to the agent. Crucially, the true label is always revealed to the agent {\em after} the decision has been made, regardless of whether the agent opts to abstain.
Their findings suggest that equipping the agent with an abstention option can significantly improve the guarantees on the worst-case regret. 


\section{Problem Setup}
\label{section_model}
\paragraph{Multi-armed bandits with abstention.} We consider a $K$-armed bandit model, enhanced with an additional option to abstain from accepting the stochastic instantaneous reward prior to its observation. Let $\mu\in\mathcal U := \mathbb R^K$ denote a specific bandit instance, where $\mu_i$ represents the unknown mean reward of arm $i\in[K]$. We assume that arm $1$ is the unique optimal arm, i.e., $1 = \argmax _{i\in [K]} \mu_i $, and we define $\Delta_i := \mu_1 - \mu_i$ as the suboptimality gap for each arm~$i$.

At each time step $t\in\mathbb N$, the agent chooses an arm $A_t\in [K]$, and, simultaneously, decides whether or not to abstain, indicated by a binary variable $B_t$. Regardless of the decision to abstain, the agent observes a random variable $X_t$ from the selected arm $A_t$, which is drawn from a Gaussian distribution $\mathcal N (\mu_{A_t}, 1)$   independent of  observations obtained from the previous time steps. Notably, the selection of both $A_t$ and $B_t$ might depend on the previous decisions and observations, as well as on each other. Formally, let $\mathcal F_t :=\sigma(A_1,B_1,X_1,\ldots,A_t,B_t, X_t)$ denote the  $\sigma$-field generated by the cumulative interaction history up to and including time $t$. It follows that the pair of random variables $(A_t, B_t)$ is $\mathcal F_{t-1}$-measurable.

The instantaneous regret at time $t$ is  determined by both the binary abstention variable $B_t$ and the observation $X_t$. Based on the outcome of the abstention option, we now discuss two complementary settings. 
In the \emph{fixed-regret} setting, the abstention option incurs a constant regret.  Opting for abstention ($B_t = 1$) leads to a deterministic regret of $c > 0$, in contrast to the initial regret linked to arm $A_t$ when not selecting abstention ($B_t = 0$), which is given by $\mu_1 - X_t$.

In the \emph{fixed-reward} setting, the reward of the abstention option is  $c\in \mathbb R$.\footnote{With a slight abuse of notation, we use the symbol $c$ to represent the abstention regret and the abstention reward within their respective settings. The surrounding context should elucidate the exact meaning of~$c$.} Since the abstention reward $c$ may be larger than $\mu_1$, the highest expected reward at each time step is $\mu_1\vee c := \max\{ \mu_1, c\}$.
If the agent decides to abstain ($B_t = 1$), it is guaranteed a deterministic reward of $c$, leading to a regret of $\mu_1\vee c-c$. If $B_t=0$, the agent receives a reward of $X_t$, resulting in a regret of $\mu_1\vee c- X_t$.

\begin{figure*}[t]
\centering
{
\renewcommand{\baselinestretch}{1}
\normalsize
\begin{tikzpicture}[>=stealth, node distance=1cm, auto]
\usetikzlibrary{positioning, arrows, shapes,calc}
\usetikzlibrary{decorations.markings}
    \tikzstyle{block} = [rectangle, draw, align=center, font=\normalsize, rounded corners, fill=blue!10, minimum width=1.4cm, minimum height=0.9cm]
    \tikzstyle{arrow} = [->, thick]
    \tikzstyle{doublearrow} = [thick, decoration={markings,mark=at position 1 with {\arrow[semithick]{open triangle 60}}},
   double distance=2pt, shorten >= 5.5pt,
   preaction = {decorate},
   postaction = {draw,line width=1.4pt, white,shorten >= 4.5pt}]

    \node[block] (Agent) {Agent};
    \node[block, above right=0.6cm and 0.35cm of Agent] (Action) {Action};
    \node[block, right=1.45cm of Agent] (Environment) {Environment};
    \node[block, right=0.3cm of Environment, fill=orange!10] (Regret) {Instantaneous\\ Regret};

    \node[inner sep=0, right=0.85cm of Regret, align=left, fill=orange!10] (Table) {\renewcommand{\arraystretch}{2.8}\setlength\extrarowheight{-1.5pt}\begin{tabular}{|c|c|c|}
        \hline
        & $B_t=0$ & $B_t=1$ \\ 
        \hline 
        \makecell{Fixed-\\ Regret} & $\mu_1-X_t$ & $c$ \\
        \hline
        \makecell{Fixed-\\ Reward} & $\!\mu_1\!\vee\! c\!-\!X_t\!$ & $\!\mu_1\!\vee\! c\!-\!c\!$ \\
        \hline
    \end{tabular}};

    \draw[arrow] (Agent) |- node[pos=0.72, above] {\(A_t, B_t\)} (Action);
    \draw[arrow] (Action) -| node[midway, near end, right] {\(A_t\)}(Environment);
    \draw[arrow] (Action) -| node[midway, near end, right] {\(B_t\)}(Regret);
    \draw[arrow] (Environment) -- ++(0, -1.5cm) -| node[midway, near end, right] {\(X_t\)} (Agent);
    \draw[arrow] (Environment) -- ++(0, -1.5cm) -| node[midway, near end, right] {\(X_t\)}(Regret);
    \draw[doublearrow] ($(Regret.east) + (0.1cm, 0)$)  -- ($(Table.west) +(-0.05cm, 0)$);
\end{tikzpicture}
}
\caption{Interaction protocol for multi-armed bandits with fixed-regret and fixed-reward abstention.}
\label{figure_model}
\vspace{0pt}
\end{figure*}
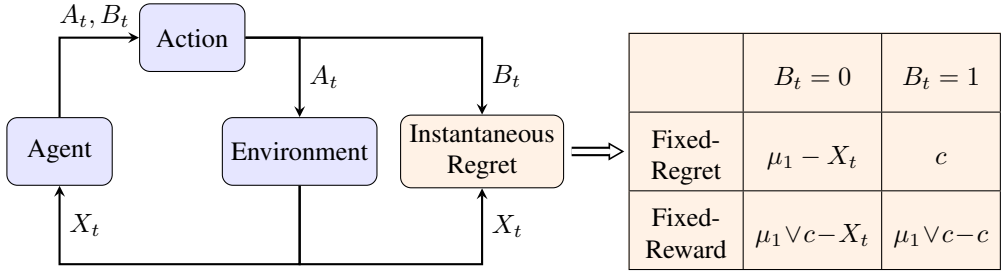

See Figure~\ref{figure_model} for a schematic of our model in the two settings.

\begin{remark}
We note that if the agent cannot observe the sample from the chosen arm when opting for abstention, then it is equivalent to skipping the time step, rendering the model trivial. In other words, the learner gains no information from the time step but incurs an instantaneous regret. In view of this, our model well aligns with that of \citet{neu2020fast}, which explored the role of abstention in the context of online prediction with expert advice.
\end{remark}

\begin{remark}[Graph Feedback]
Bandits with graph feedback \citep{liu2018information,alon2017nonstochastic} are a natural extension of the standard multi-armed bandit setting. In this framework, each arm is represented as a node in a graph, and pulling one arm reveals not only its own reward but also observations from the neighboring nodes connected to it in the graph.  Our model bears a conceptual resemblance to bandits with graph feedback: it can be represented as a bipartite graph in which each stochastic arm $a$ is associated with a corresponding abstention arm $Q_a$. Pulling a stochastic arm $a$ reveals a reward drawn from arm $a$'s distribution, while choosing the corresponding abstention arm $Q_a$ yields an observation drawn from the reward distribution of the associated stochastic arm as well.

That said, there is a key structural difference between the two frameworks. In bandits with graph feedback, selecting an arm reveals not only the reward of the chosen arm but also observations from its neighboring arms. In contrast, in our model, choosing an abstention arm does not provide any information about the abstention arm itself. Instead, the mean reward (or regret) of each abstention arm is known to the learner at the beginning of the game.
\end{remark}

\paragraph{Regret minimization.} 
Our overarching goal is to design and analyze online algorithms $\pi$ that minimize their expected cumulative regrets up to and including the time horizon $T$.\footnote{In certain real-world applications, the time horizon $T$ may be unknown to the agent. In fact, all of our proposed methods are inherently \emph{anytime} in nature, as they do not necessitate prior knowledge of the horizon.} The regrets are formally defined  for the two distinct settings as follows:

\setul{}{0.4pt} 
\setuldepth{3pt} 
\vspace{6pt}\noindent\labelitemi \ Fixed-\ul{r}e\ul{g}ret setting:
\vspace{-6pt}
\begin{equation}
\begin{aligned}
    R^{\FRG}_{\mu,c}(T, \pi) :=\mathbb E\Bigg[&\sum_{t=1}^T\big ((\mu_1-X_t)\cdot\mathbf 1\{B_t=0\}  +c\cdot\mathbf 1\{B_t=1\}\big )\Bigg] .
    \label{equation_fixedregret_regret}
\end{aligned}
\end{equation}

\vspace{-6pt}\noindent\labelitemi \ Fixed-\ul{r}e\ul{w}ard setting:
\vspace{-6pt}
\begin{equation}
\begin{aligned}
R^{\FRW}_{\mu,c}(T, \pi) :=T \!\cdot \! ( \mu_1\vee c )- \mathbb E\Bigg[&\sum_{t=1}^T\Big(X_t\cdot\mathbf 1\{B_t=0\} +c\cdot\mathbf 1\{B_t=1\}\Big)\Bigg] .
\end{aligned}
\label{equation_fixedreward_regret}
\end{equation}



An online algorithm $\pi$ consists of two interrelated components:  the \emph{arm sampling} rule that selects $A_t$, and the \emph{abstention decision rule} that determines $B_t$ at each time step $t\in [T]$.
Additionally, we use $\Pi^{\FRG}$ and $\Pi^{\FRW}$ to denote the collections of all online policies for the fixed-regret and fixed-reward settings, respectively. For the sake of analytical convenience, we also introduce the \underline{ca}nonical regret $R^{\CA}_{\mu}(T, \pi) := T\mu_1 - \mathbb E\big[\sum_{t=1}^T X_t \big]$, which disregards the abstention option and remains well-defined within our abstention model. Furthermore, when there is no ambiguity, we will omit the dependence of the regret on the policy. For example, we  often abbreviate $R^{\FRG}_{\mu,c}(T, \pi)$ as $R^{\FRG}_{\mu,c}(T)$.

\vspace{0pt}
\begin{remark}
It is worth mentioning that our model is a strict generalization of the canonical multi-armed bandit model (without the abstention option). Specifically, it particularizes to the canonical model as the abstention regret $c$ tends to  $+\infty$ in the fixed-regret setting and as the abstention reward $c$ tends to $-\infty$ in the fixed-reward setting. Nevertheless, the incorporation of an extra challenge, the abstention decision (denoted as $B_t$), offers the agent the potential opportunity to achieve superior performance in terms of either regret.  {One might wonder whether our setting reduces to a $(K+1)$-armed bandit model  where the $(K+1)$-st arm models abstention. It does not; for a detailed discussion on why our framework cannot be reduced to a standard bandit problem with an additional deterministic arm, the reader is referred to Appendix~\ref{appendix_comparison_k_plus_1}.}
\end{remark}

\paragraph{Other notations.}  For $x,y\in \mathbb R$, we denote $x \wedge y :=\min \{x, y\}$ and $x\vee y :=\max \{x, y\}$.  For any arm $i \in [K]$, let $N_i(t) := \sum_{s=1}^t \mathbf 1 \{ A_{s} = i \}$ and $\hat \mu_i (t):= \sum_{s=1}^t X_{s}\mathbf 1 \{ A_{s} =i \} / N_i(t)$ denote its total number of  pulls and empirical estimate of the mean up to time $t$, respectively. {In particular, we set $\hat{\mu}_i(t) = +\infty $ if $N_i(t)= 0$.} 
To count abstention records, we also use $N_i^{(0)}(t)$ and $N_i^{(1)}(t)$ to denote its number of pulls without and with abstention up to time $t$, respectively. That is, $N_i^{(0)}(t) := \sum_{s=1}^t \mathbf 1 \{ A_{s} = i \text { and }  B_s = 0\}$ and $N_i^{(1)}(t) := \sum_{s=1}^t \mathbf 1 \{ A_{s} = i \text { and }  B_s = 1\}$. 
Additionally, we define $\hat{\mu}_{is}$ as the empirical mean of arm $i$ based on its first $s$ pulls.
Furthermore, we use $\alpha$, $\alpha_1$, and so forth to represent universal constants that do not depend on the \emph{problem instances} (including $\mu$, $c$, $T$, $K$), with possibly different values in different contexts.


\section{Fixed-Regret Setting}
\label{section_fixedregret}  
In this section, we focus on the fixed-regret setting. Specifically, we design a conceptually simple and computationally efficient algorithm, namely \underline{F}ixed-\underline{R}e\underline{g}ret \underline{T}hompson \underline{S}ampling \underline{w}ith \underline{A}bstention (or \textsc{FRG-TSwA}), to minimize the cumulative regret while incorporating  fixed-regret abstention. To evaluate the performance of our algorithm from a theoretical standpoint, we establish  both instance-dependent asymptotic and instance-independent minimax upper bounds on the cumulative regret, as elaborated upon in Section~\ref{subsection_fixedregret_upperbound}. Furthermore, in Section~\ref{subsection_fixedregret_lowerbound}, we provide lower bounds for the problem of regret minimization in multi-armed bandits with fixed-regret abstention. These findings substantiate that our algorithm achieves both asymptotic and minimax optimality simultaneously.
The pseudocode for \textsc{FRG-TSwA} is presented in Algorithm~\ref{algo1} and elucidated in the following. 

\begin{algorithm}[t]
\caption{Fixed-Regret Thompson Sampling with Abstention (or \textsc{FRG-TSwA})}
\label{algo1}
\hspace*{0.02in} {\bf Input:} Arm set $[K]$ and abstention regret $c>0$.

\begin{algorithmic}[1]
\STATE Sample each arm once, and choose to abstain ($B_t = 1$) if and only if $\sqrt{\frac K t} \ge c$.
\STATE Initialize $\hat{\mu}_i(K)$ and $N_i(K)=1$ for all $i \in[K]$.
\FOR{$t= K+1,\ldots,T$}
\STATE For each arm $i\in[K]$, sample $\theta_i(t) \sim \mathcal N (\hat{\mu}_i(t-1), 1/N_i(t-1))$ and set 
$$
a_i(t)= \begin{cases}\theta_i(t) & \text { with probability } 1/K \\ \hat{\mu}_i(t-1) & \text { with probability } 1-1/K. \end{cases}
$$
\STATE Pull the arm $A_t = \argmax_{i\in[K]} a_i(t)$, and choose to abstain ($B_t = 1$) if and only if 
$$
\max_{i\in [K]\setminus \{A_t\}} \mathrm{LCB}_i(t) - \hat{\mu}_{A_t}(t-1) \ge c \ \text{ or }\ \sqrt{\frac K t} \ge c.
$$
\STATE Observe $X_t$ from the arm $A_t$, and update $\hat{\mu}_i(t)$ and $N_i(t)$ for all $i \in[K]$.
\ENDFOR
\end{algorithmic}
\end{algorithm}

In terms of the arm sampling rule, our algorithm is built upon Less-Exploring Thompson Sampling \citep{jin23b}, a minimax optimal enhancement of the celebrated Thompson Sampling (TS) algorithm \citep{thompson1933likelihood}. We refer to Remark~\ref{remark_TSreason} for the reason behind this choice. During the initialization phase, each arm is sampled exactly once. Following that, at each time $t$, an estimated reward $a_i(t)$ is constructed for each arm $i\in[K]$, which is either drawn from the posterior distribution $\mathcal N (\hat{\mu}_i(t-1), 1/N_i(t-1))$ with probability $1/K$ or set to be the empirical mean $\hat{\mu}_i(t-1)$ otherwise. Subsequently, the algorithm consistently pulls the arm $A_t$ with the highest estimated reward.  

To pursue the dual optimality properties, it is essential to carefully balance the simultaneous control of asymptotic and worst-case regrets resulting from potential misjudgments in choosing the abstention option. With regard to the abstention decision rule, we propose two abstention criteria that work in tandem (as detailed in Step~5 of Algorithm~\ref{algo1}). The first criterion is gap-dependent in nature. In particular, for each arm $i\in [K]$, we define its \emph{lower confidence bound} as
\begin{equation}
\mathrm{LCB}_i(t) := \hat{\mu}_{i}(t-1)-\sqrt{\frac{6\log t+2\log(c\vee 1)}{N_{i}(t-1)}}.
\label{equation_LCB}
\end{equation}
Then we choose to abstain if there exists an arm $i\in [K]\setminus \{A_t\}$ for which the difference between $\mathrm{LCB}_i(t)$ and the empirical mean of the arm $A_t$ exceeds $c$. This condition signifies that the suboptimality gap $\Delta_{A_t}$ is at least $c$ with high probability. {The LCB in~\eqref{equation_LCB} is custom-tailored for our specific application, which is unique in the fact that it takes the value of the abstention regret $c$ into consideration. Notably, besides its asymptotic optimality, our choice to employ this LCB rather than other natural types of upper or lower confidence bounds in the abstention criteria (e.g., those that do not include $c$) is primarily due to the imperative to constrain the minimax regret to be $O(\sqrt{KT})$.}

The second abstention criterion is gap-independent. It is motivated from the construction of worst-case scenarios as detailed in the proof of our lower bound. Under this criterion, we opt for the abstention option if $c\le\sqrt{ K / t}$, which implies that the abstention regret remains acceptably low at time $t$ in view of the  worst-case scenarios.

\begin{remark}
\label{remark_TSreason}
    As previously highlighted, our model in the fixed-regret setting particularizes to the canonical multi-armed bandit model as the abstention regret $c$ tends to infinity. Similarly, when $c$ tends to infinity, the two abstention criteria are never satisfied, and the procedure of Algorithm~\ref{algo1} simplifies to that of Less-Exploring Thompson Sampling. It is worth noting that this latter algorithm is not only asymptotically optimal but also minimax optimal for the canonical model. This is precisely why we base our algorithm upon it, rather than the conventional Thompson Sampling algorithm, which  has been shown not to be minimax optimal \citep{agrawal2017near}. 
\end{remark}

\subsection{Upper Bounds}
\label{subsection_fixedregret_upperbound}
Theorem~\ref{theorem_fixedregret_upperbound} below provides two distinct types of theoretical guarantees pertaining to our algorithm's performance on the cumulative regret ${R^{\FRG}_{\mu,c}(T)}$, defined in Equation~\eqref{equation_fixedregret_regret} for the fixed-regret setting.

\begin{theorem}
\label{theorem_fixedregret_upperbound}
For  all abstention regrets $c>0$ and  bandit instances $\mu\in\mathcal U$, Algorithm~\ref{algo1} guarantees that 
$$
\limsup_{T\to \infty} \frac {R^{\FRG}_{\mu,c}(T)} {\log T} \le  2 \sum_{i>1} \frac{\Delta_i\wedge c }{\Delta_i^2} .
$$
Furthermore, there exists a universal constant $\alpha > 0$ such that
$$
{R^{\FRG}_{\mu,c}(T)}  \le  \begin{cases}
     cT  & \text{ if }  c\le \sqrt{ K/ T} \\
   \alpha  (\sqrt{KT}+\sum_{i>1}{\Delta_i}) & \text{ if }  c> \sqrt{ K /T}.
\end{cases}
$$
\end{theorem}

The   proof of Theorem~\ref{theorem_fixedregret_upperbound} is deferred to Appendix~\ref{subsection_fixedregret_upperboundproof}. The theoretical challenges arising from proving Theorem~\ref{theorem_fixedregret_upperbound} revolve around the quantification of the regret that results from inaccurately estimating the suboptimality gaps associated with the abstention criteria. More precisely, it is crucial to establish upper bounds on $\E[N_i^{(1)}(T)]$ for arms $i$ with $\Delta_i <c$ (which, by definition, includes the best arm), and on $\E[N_i^{(0)}(T)]$ for arms $i$ with $\Delta_i >c$. 
These bounds need to be derived from both asymptotic and minimax perspectives, adding layers of complexity to the analytical process. To accomplish these tasks, we initially decompose the target expectations into manageable components, and subsequently bound the resulting subterms in the two respective regimes. Our approach, merging TS-based sampling rule with LCB-based abstention criteria, necessitates a careful amalgamation of both TS-type and UCB-type analytical techniques.

Furthermore, for the asymptotic upper bound, we take a deeper exploration into the randomized arm sampling dynamics inherent to Less-Exploring Thompson Sampling. A pivotal aspect of this exploration is a characterization of a high probability lower bound on the number of pulls of the optimal arm, as detailed in Lemma~\ref{lemma_N1t}, inspired by the work of \citet{korda2013thompson}.


Finally, we remark that the numerous intricacies involved in these analyses preclude us from formulating a generalized strategy akin to the forthcoming Algorithm~\ref{algo2} for the fixed-reward setting.

\subsection{Lower Bounds}
\label{subsection_fixedregret_lowerbound}
In order to establish the asymptotic lower bound, we need to introduce the concept of $R^{\FRG}$-consistency, which rules out overly specialized algorithms that are tailored exclusively to specific problem instances. Roughly speaking, a $R^{\FRG}$-consistent algorithm guarantees a subpolynomial cumulative regret for any given problem instance.

\begin{definition}[$R^{\FRG}$-consistency]
\label{definition_fixedregret_consistent}
An algorithm $\pi\in \Pi^{\FRG}$ is said to be \emph{$R^{\FRG}$-consistent} if for all abstention regrets $c>0$,  bandit instances $\mu\in\mathcal U$, and $a>0$, $R^{\FRG}_{\mu,c}(T, \pi)=o(T^a)$.
\end{definition}

Now we present both asymptotic and minimax lower bounds on the cumulative regret in Theorem~\ref{theorem_fixedregret_lowerbound}, which is proved in Appendix~\ref{subsection_fixedregret_lowerboundproof}.

\begin{theorem} 
\label{theorem_fixedregret_lowerbound}
For any abstention regret $c>0$, bandit instance $\mu\in\mathcal U$ and $R^{\FRG}$-consistent algorithm~$\pi$, it holds that
$$
\liminf_{T\to \infty} \frac {R^{\FRG}_{\mu,c}(T,\pi)} {\log T} \ge  2 \sum_{i>1} \frac{\Delta_i\wedge c }{\Delta_i^2} . 
$$
For any abstention regret $c>0$ and time horizon $T\ge K$, there exists a universal constant $\alpha > 0$ such that
$$
\inf_{\pi\in \Pi^{\FRG}} \sup _{\mu\in \mathcal U} {R^{\FRG}_{\mu,c}(T,\pi)}  \ge \alpha (\sqrt{KT}\wedge cT).
$$
\end{theorem}

Comparing the upper bounds on the cumulative regret of our algorithm \textsc{FRG-TSwA} in Theorem~\ref{theorem_fixedregret_upperbound} with the corresponding lower bounds in Theorem~\ref{theorem_fixedregret_lowerbound}, it is evident that our algorithm exhibits both asymptotic and minimax optimality.

\paragraph{Asymptotic optimality.}  For  any abstention regret $c>0$ and  bandit instance $\mu \in\mathcal U$, the regret of our algorithm satisfies the following limiting behaviour:
$$
\lim_{T\to \infty} \frac {R^{\FRG}_{\mu,c}(T)} {\log T} =  2 \sum_{i>1} \frac{\Delta_i\wedge c }{\Delta_i^2} .
$$
The above asymptotically optimal result yields several intriguing implications.  First, the inclusion of the additional fixed-regret abstention option does not obviate the necessity of differentiating between suboptimal arms and the optimal one, and the exploration-exploitation trade-off remains crucial. In fact, to avoid the case in which the cumulative regret grows polynomially, the agent must still asymptotically allocate the same proportion of pulls to each suboptimal arm, as in the canonical model. This assertion is rigorously demonstrated in the proof of the lower bound (refer to Appendix~\ref{subsection_fixedregret_lowerboundproof} for details). Nevertheless, the abstention option does indeed reduce the exploration cost for the agent. Specifically, when exploring any suboptimal arm with a suboptimality gap larger than $c$, our algorithm tends towards employing the abstention option to minimize the instantaneous regret. This aspect is formally established in the proof of the asymptotic upper bound (see Appendix~\ref{subsection_fixedregret_upperboundproof}).

\paragraph{Minimax optimality.} In the context of worst-case guarantees for the cumulative regret, we focus on the dependence on $c$, $K$, and $T$. Notably, the $\sum_{i>1}{\Delta_i}$ term\footnote{This term is unavoidable when the abstention regret $c$ is sufficiently high, since every reasonable algorithm has to allocate a fixed number of pulls to each arm.} is typically considered as negligible in the literature \citep{audibert2009minimax, agrawal2017near, lattimore2020bandit}. Therefore, Theorem~\ref{theorem_fixedregret_upperbound} demonstrates that our algorithm attains a worst-case regret of $O(\sqrt{KT}\wedge cT)$, which is minimax optimal in light of Theorem~\ref{theorem_fixedregret_lowerbound}.

A phase transition phenomenon can be clearly observed from the worst-case guarantees, which dovetails with our intuitive understanding of the fixed-regret abstention setting. When the abstention regret $c$ is sufficiently low, it becomes advantageous to consistently opt for abstention to avoid the worst-case scenarios. On the contrary, when the abstention regret $c$ exceeds a certain threshold, the abstention option proves to be inadequate in alleviating the worst-case regret, as compared to the canonical model.

\begin{remark}
\label{remark_linear}
Although our model allows for the selected arm $A_t$ and the abstention option $B_t$ to depend on each other, the procedure used in both algorithms within this work is to first determine $A_t$ {\em before} $B_t$; this successfully achieves both forms of optimality. 
Nevertheless, this approach might no longer be optimal beyond the canonical $K$-armed bandit setting. In $K$-armed bandits, each arm operates independently. Conversely, in models like linear bandits, pulling one arm can indirectly reveal information about other arms. Policies based on the principle of optimism in the face of uncertainty, as well as Thompson Sampling, fall short of achieving asymptotic optimality in the context of linear bandits \citep{lattimore2017end}.
Therefore, the abstention option becomes particularly attractive if there exists an arm that incurs a substantial regret but offers significant insights into the broader bandit instance.
\end{remark}

\paragraph{Stochastic and heterogeneous abstention regret.} If the regret incurred by the agent when opting for abstention is not deterministic but stochastically generated from a distribution with known expectation $c$, Algorithm~\ref{algo1} remains effective, and our analyses of the upper and lower bounds remain valid. Additionally, another natural extension of our model is that the abstention regret is different for different arms. Specifically, the agent incurs a regret of $c_i>0$ when pulling arm $i$ and choosing abstention. In this scenario, the abstention criteria in Algorithm~\ref{algo1} can be adapted accordingly, and the same asymptotic and minimax optimality can be established; see Appendix~\ref{appendix_hetorogenous_abstention} for details.

\color{black}

\section{Fixed-Reward Setting}
\label{section_fixedreward}
In this section, we investigate the fixed-reward setting. Here, the reward associated with the abstention option remains consistently fixed at $c\in \mathbb R$. When exploring a specific arm, the agent has the capability to determine whether selecting the abstention option yields a higher reward   \emph{solely} based on its own estimated mean reward. However, in the fixed-regret setting, this decision can only be made by taking into account \emph{both} its own estimated mean reward and the estimated mean reward of the potentially best arm. In this regard, the fixed-reward setting is inherently less complex than the fixed-regret setting. As a result, it becomes possible for us to design a more general strategy \underline{F}ixed-\underline{R}e\underline{w}ard \underline{Alg}orithm \underline{w}ith \underline{A}bstention (or \textsc{FRW-ALGwA}), whose pseudocode is presented in Algorithm~\ref{algo2}. Despite the straightforward nature of our algorithm, we demonstrate its dual attainment of both asymptotic and minimax optimality through an exhaustive theoretical examination in Sections~\ref{subsection_fixedreward_upperbound} and \ref{subsection_fixedreward_lowerbound}.

\begin{algorithm}[t]
\caption{Fixed-Reward Algorithm with Abstention (or \textsc{FRW-ALGwA})}
\label{algo2}
\hspace*{0.02in} {\bf Input: } Arm set $[K]$, abstention reward $c\in \mathbb R$, and a base algorithm \textsc{ALG} that is both asymptotically and minimax optimal for the canonical multi-armed bandit model$.$
\begin{algorithmic}[1]
\STATE Initialize $\hat{\mu}_i(0)=+\infty$ for all arms $i \in[K]$.
\FOR{$t= 1, 2,\ldots,T$}
\STATE Pull the arm $A_t$ chosen by the base algorithm \textsc{ALG}.
\STATE Choose to abstain ($B_t = 1$) if and only if $\hat{\mu}_{A_t}(t-1) \le c.$
\STATE Observe $X_t$ from the arm $A_t$, and update $\hat{\mu}_i(t)$ for all $i \in[K]$.
\ENDFOR
\end{algorithmic}
\end{algorithm}

\textsc{FRW-ALGwA} leverages a base algorithm \textsc{ALG} that is asymptotically and minimax optimal for canonical multi-armed bandits as its input. For comprehensive definitions of asymptotic and minimax optimality within the canonical model, see Appendix~\ref{appendix_standard_bandits}. Notably, eligible candidate algorithms include $\text{KL-UCB}^{++}$ \citep{menard2017minimax}, ADA-UCB \citep{lattimore2018refining}, MOTS-$\mathcal{J}$ \citep{jin2021mots} and Less-Exploring Thompson Sampling \citep{jin23b}. 
In our algorithm, at each time step $t$, the base algorithm determines the selected arm $A_t$ according to the partial interaction historical information $(A_1,X_1,A_2,X_2, \ldots,A_{t-1}, X_{t-1})$. Subsequently, the algorithm decides whether or not to abstain, indicated by the binary random variable $B_t$, by comparing the empirical mean of the arm $A_t$, denoted as $\hat{\mu}_{A_t}(t-1)$, to the abstention reward~$c$.

\subsection{Upper Bounds}
\label{subsection_fixedreward_upperbound}
Recall the definition of the cumulative regret ${R^{\FRW}_{\mu,c}(T)}$, as presented in Equation~\eqref{equation_fixedreward_regret} for the fixed-reward setting.  
Theorem~\ref{theorem_fixedreward_upperbound} establishes both the instance-dependent asymptotic and instance-independent minimax upper bounds for Algorithm~\ref{algo2}; see Appendix~\ref{subsection_fixedreward_upperboundproof} for the proof.

\begin{theorem}
\label{theorem_fixedreward_upperbound}
For  all abstention rewards $c\in\mathbb R$ and  bandit instances $\mu\in\mathcal U$, Algorithm~\ref{algo2} guarantees that 
$$
\limsup_{T\to \infty} \frac {R^{\FRW}_{\mu,c}(T)} {\log T} \le  2 \sum_{i>1} \frac{ \mu_1\vee c - \mu_i\vee c }{\Delta_i^2}.
$$
Furthermore, there exists a universal constant $\alpha > 0$ such that
$$
{R^{\FRW}_{\mu,c}(T)}  \le \alpha \left(\sqrt{KT}+\sum_{i\in[K]}\left( \mu_1\vee c -\mu_i\right)\right).
$$
\end{theorem}

\begin{remark}
\label{remark_c_large}
It is worth considering the special case where $c\ge \mu_1$, where opting for abstention results in a reward even greater than, or equal to, the mean reward of the best arm. For this particular case, as per Theorem~\ref{theorem_fixedreward_upperbound}, since $\mu_1\vee c - \mu_i\vee c=0$ for all $i > 1$, our algorithm achieves a regret of $o(\log T)$.  This result, in fact, is not surprising. In contrast to the fixed-regret setting where the regret associated with the abstention option is strictly positive, in this specific scenario of the fixed-reward setting, selecting the abstention option is indeed the optimal action at a single time step, regardless of the arm pulled. Therefore, there is no necessity to distinguish between suboptimal arms and the optimal one, and the exploration-exploitation trade-off becomes inconsequential. However, when the abstention reward is below the mean reward of the best arm, i.e., $c<\mu_1$, maintaining a subpolynomial cumulative regret still hinges on the delicate balance between exploration and exploitation, as evidenced by the forthcoming exposition of the asymptotic lower bound.
\end{remark}

\subsection{Lower Bounds}
\label{subsection_fixedreward_lowerbound}
We hereby introduce the concept of $R^{\FRW}$-consistency for the fixed-reward setting, in a manner analogous to the fixed-regret setting.
Following this, we present two distinct lower bounds for the problem of regret minimization in multi-armed bandits with fixed-reward abstention in Theorem~\ref{theorem_fixedreward_lowerbound}. The proof for Theorem~\ref{theorem_fixedreward_lowerbound} is postponed to Appendix~\ref{subsection_fixedreward_lowerboundproof}.

\begin{definition}[$R^{\FRW}$-consistency]
\label{definition_fixedreward_consistent}
An algorithm $\pi\in \Pi^{\FRW}$ is said to be \emph{$R^{\FRW}$-consistent} if for all abstention rewards $c\in \mathbb R$,  bandit instances $\mu\in\mathcal U$, and $a>0$, $R^{\FRW}_{\mu,c}(T, \pi)=o(T^a)$.
\end{definition}

\begin{theorem} 
\label{theorem_fixedreward_lowerbound}
For any abstention reward $c\in \mathbb R$,  bandit instance $\mu\in \mathcal U$ and  $R^{\FRW}$-consistent algorithm~$\pi$, it holds that
$$
\liminf_{T\to \infty} \frac {R^{\FRW}_{\mu,c}(T, \pi)} {\log T} \ge  2 \sum_{i>1} \frac{ \mu_1\vee c - \mu_i\vee c }{\Delta_i^2}.
$$
For any abstention reward $c\in \mathbb R$ and time horizon $T\ge K$, there exists a universal constant $\alpha > 0$ such that
$$
\inf_{\pi\in \Pi^{\FRW}} \sup _{\mu\in \mathcal U} {R^{\FRW}_{\mu,c}(T, \pi)}  \ge \alpha \sqrt{KT}.
$$
\end{theorem}

By comparing the upper bounds in Theorem~\ref{theorem_fixedreward_upperbound} with the lower bounds in Theorem~\ref{theorem_fixedreward_lowerbound}, it is firmly confirmed that Algorithm~\ref{algo2} is both asymptotically and minimax optimal in the fixed-reward setting.

\paragraph{Asymptotic optimality.}  For any abstention reward $c\in\mathbb R$ and  bandit instance $\mu\in\mathcal U$, our algorithm ensures the following optimal asymptotic behavior for the cumulative regret:
$$
\lim_{T\to \infty} \frac {R^{\FRW}_{\mu,c}(T)} {\log T} = 2 \sum_{i>1} \frac{ \mu_1\vee c - \mu_i\vee c }{\Delta_i^2}.
$$
Since it holds  that $\mu_1\vee c - \mu_i\vee c\le \Delta_i$ for all $i > 1$, our algorithm effectively reduces the cumulative regret in the asymptotic regime  through the incorporation of the fixed-reward abstention option.

\paragraph{Minimax optimality.} Concerning the worst-case performance of our algorithm, disregarding the additive term $\sum_{i\in[K]}\left( \mu_1\vee c -\mu_i\right)$, it achieves an optimal worst-case regret of $O(\sqrt{KT})$. While this worst-case regret is the same as that for canonical multi-armed bandits, achieving it is non-trivial as we have to  simultaneously achieve the minimal asymptotic instance-dependent regret. 


Moreover, there is no occurrence of the phase transition phenomenon in the fixed-reward setting. This absence can be attributed to the intrinsic nature of the fixed-reward abstention option. For any abstention reward $c\in\mathbb R$ and online algorithm, we can always construct a challenging bandit instance that leads to a cumulative regret of $\Omega(\sqrt{KT})$, as demonstrated in the proof of the minimax lower bound in Appendix~\ref{subsection_fixedreward_lowerboundproof}.

\section{Numerical Experiments}
\label{section_experiment}
In this section, we conduct numerical experiments to  empirically substantiate our theoretical insights. To reduce clutter, we report our results only for the fixed-regret setting here. Results pertaining to the fixed-reward setting are deferred to Appendix~\ref{appendix_experiment}, where we consider two particular realizations of our algorithm based on Less-Exploring Thompson Sampling \citep{jin23b} and $\text{KL-UCB}^{++}$ \citep{menard2017minimax}. In each experiment, the reported cumulative regrets are averaged over $2,000$ independent trials and the corresponding standard deviations are displayed as error bars in the figures.



To confirm the benefits of incorporating the abstention option, we compare the performance of our proposed algorithm \textsc{FRG-TSwA} (Algorithm~\ref{algo1}) with that of Less-Exploring Thompson Sampling \citep{jin23b}, which serves as a  baseline algorithm without the abstention option. We consider two synthetic bandit instances. The first instance $\mu^{\dagger}$ with $K=7$ has uniform suboptimality gaps: $\mu^{\dagger}_1 = 1$ and $\mu^{\dagger}_i = 0.7$ for all $i\in[K]\setminus \{1\}$. For the second instance $\mu^{\ddagger}$ with $K=10$, the suboptimality gaps are more diverse: $\mu^{\ddagger}_1 = 1$, $\mu^{\ddagger}_i = 0.7$ for $i \in \{ 2,3,4\}$, $\mu^{\ddagger}_i = 0.5$ for $i \in\{ 5,6,7\}$ and $\mu^{\ddagger}_i = 0.3$ for $i \in\{ 8,9,10\}$. The empirical averaged cumulative regrets of both methods with abstention regret $c=0.1$ for different time horizons $T$ are presented in Figure~\ref{fig_FRG_main1}. 
To demonstrate their asymptotic behavior, we also plot the instance-dependent asymptotic lower bound on the cumulative regret (refer to Theorem~\ref{theorem_fixedregret_lowerbound}) in each sub-figure. It can be observed that \textsc{FRG-TSwA} is clearly superior compared to the non-abstaining baseline, especially for large values of $T$. This demonstrates the advantage of the abstention mechanism. With regard to the growth trend, as the time horizon $T$ increases, the curve corresponding to \textsc{FRG-TSwA} closely approximates that of the asymptotic lower bound. This suggests that the expected cumulative regret of  \textsc{FRG-TSwA} matches the lower bound asymptotically, thereby substantiating the theoretical results presented in Section~\ref{section_fixedregret}.

\begin{figure}[t]
\vspace{0pt}
 \centering
	\begin{minipage}{0.79\linewidth}
		\hspace{-6pt}
		\subfigure[Instance $\mu^{\dagger}$]{
			\begin{minipage}[b]{0.48\textwidth}
				\includegraphics[width=1.02\textwidth]{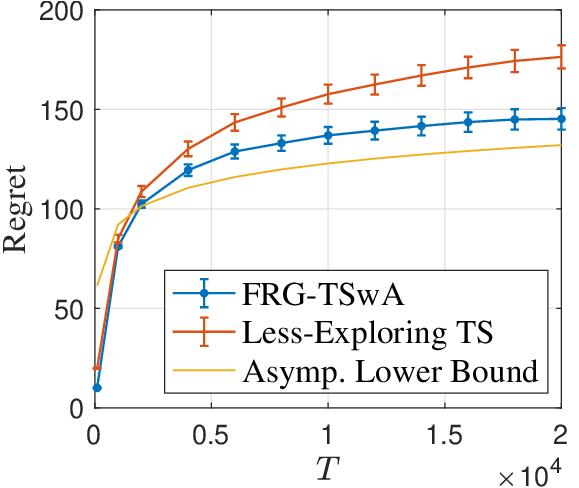}
			\end{minipage}
		}
		\hspace{-6pt}
		\subfigure[Instance $\mu^{\ddagger}$]{
			\begin{minipage}[b]{0.48\textwidth}  
				\includegraphics[width=1.02\textwidth]{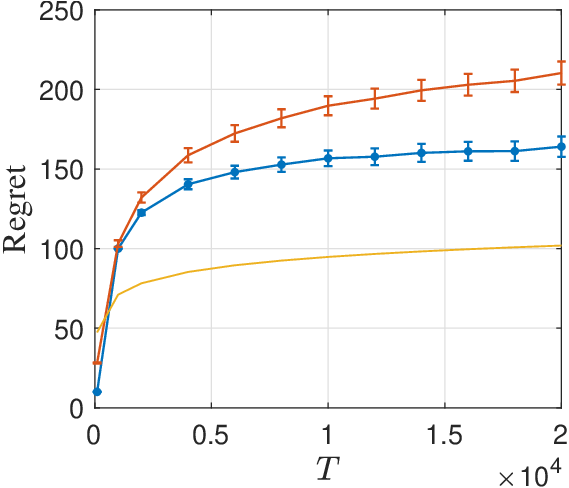}
			\end{minipage}
		}
	\end{minipage}
        \vspace{-5pt}
	\caption{Empirical regrets with abstention regret $c=0.1$ for different time horizons~$T$.}
	\label{fig_FRG_main1}
 \vspace{0pt}
\end{figure}
\begin{figure}[t]
 \vspace{0pt}
 \centering
	\begin{minipage}{0.79\linewidth}
		\hspace{-6pt}
		\subfigure[Instance $\mu^{\dagger}$]{
			\begin{minipage}[b]{0.45\textwidth}
				\includegraphics[width=1.02\textwidth]{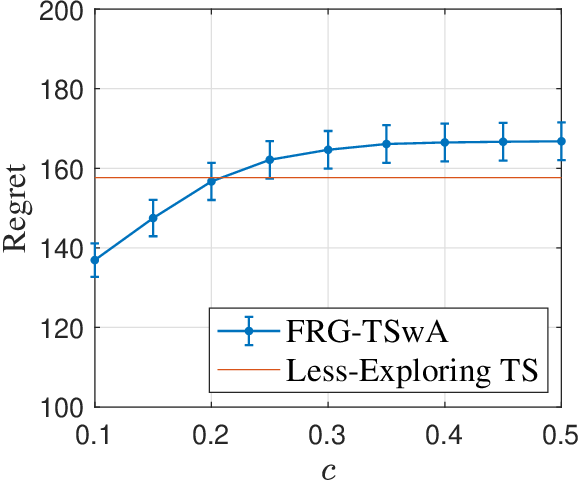}
			\end{minipage}
		}
		\hspace{-6pt}
		\subfigure[Instance $\mu^{\ddagger}$]{
			\begin{minipage}[b]{0.45\textwidth}  
				\includegraphics[width=1.02\textwidth]{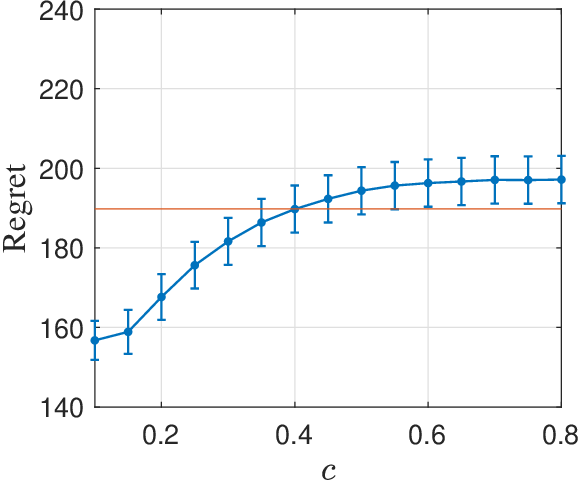}
			\end{minipage}
		}
	\end{minipage}
\vspace{-5pt}
	\caption{Empirical regrets with time horizon $T=10,000$ for different abstention regrets~$c$.}
	\label{fig_FRG_main2}
\vspace{0pt}
\end{figure}

To illustrate the effect of the abstention regret $c$, we evaluate the performance of \textsc{FRG-TSwA} for varying values of $c$,  while keeping the time horizon $T$ fixed at $10,000$. The experimental results for both bandit instances $\mu^{\dagger}$ and $\mu^{\ddagger}$ are presented in Figure~\ref{fig_FRG_main2}. Within each sub-figure, we observe that as $c$ increases, the empirical averaged cumulative regret initially increases but eventually saturates beyond a certain threshold value of $c$. These empirical observations align well with our expectations. Indeed, when provided with complete information about the bandit instance, if the abstention regret $c$ exceeds the largest suboptimality gap, the agent gains no advantage in choosing the abstention option when selecting any arm. However, we remark that the agent lacks this oracle-like knowledge of the suboptimality gaps and must estimate them on the fly.  Consequently, this results in the inevitable selection of the abstention option, even when the abstention regret $c$ is large.

\section{Conclusions and Future Work} \label{app:concl}
\label{section_conclusion}
In this paper, we consider, for the first time, a multi-armed bandit model that allows for the possibility of {\em abstaining} from accepting the stochastic rewards, alongside the conventional arm selection. This innovative framework is motivated by real-world scenarios where decision-makers may wish to hedge against highly uncertain or risky actions, as exemplified in clinical trials. Within this enriched paradigm, we address both the fixed-regret and fixed-reward settings, providing tight upper and lower bounds on  asymptotic and minimax regrets for each scenario. 
{For the fixed-regret setting, we thoughtfully adapt a recently developed asymptotically and minimax optimal algorithm by \citet{jin23b} to accommodate the abstention option while preserving its attractive optimality characteristics. For the fixed-reward setting, we convert {\em any} asymptotically and minimax optimal algorithm for the canonical model into one that retains these optimality properties when the abstention option is present.}
 {We remark that our proposed algorithms are computationally efficient, incurring only a negligible constant-time overhead per step beyond the standard arm selection mechanisms.}
Finally, experiments on synthetic datasets validate our theoretical results and clearly demonstrate the advantage of incorporating the abstention option.

 {While our analysis focuses on Gaussian rewards, the framework naturally extends to other distributions, provided that specific boundary cases are handled with care.} As highlighted in Remark~\ref{remark_linear}, a fruitful avenue for future research lies in expanding the abstention model from $K$-armed bandits to linear bandits; specifically, it remains an open question whether the inclusion of the abstention feature can lead to enhanced asymptotic and minimax theoretical guarantees. Furthermore, in our work, the abstention option exerts no influence on the stochastic observation from the selected arm.
Delving into more sophisticated and general approaches to model the effect of the abstention option promises as a captivating future direction.

\section*{Acknowledgments}
T.~Jin  and V.~Y.~F.~Tan are supported by a Singapore Ministry of Education (MOE) AcRF Tier 2 grant under grant number A-8004062-00-00.

\bibliography{reference}

\begin{thebibliography}{34}
\providecommand{\natexlab}[1]{#1}
\providecommand{\url}[1]{\texttt{#1}}
\expandafter\ifx\csname urlstyle\endcsname\relax
  \providecommand{\doi}[1]{doi: #1}\else
  \providecommand{\doi}{doi: \begingroup \urlstyle{rm}\Url}\fi

\bibitem[Agrawal \& Goyal(2012)Agrawal and Goyal]{agrawal2012analysis}
Shipra Agrawal and Navin Goyal.
\newblock Analysis of {Thompson} sampling for the multi-armed bandit problem.
\newblock In \emph{Conference on Learning Theory (COLT)}, pp.\  39--1. JMLR Workshop and Conference Proceedings, 2012.

\bibitem[Agrawal \& Goyal(2017)Agrawal and Goyal]{agrawal2017near}
Shipra Agrawal and Navin Goyal.
\newblock Near-optimal regret bounds for {Thompson} sampling.
\newblock \emph{Journal of the ACM (JACM)}, 64\penalty0 (5):\penalty0 1--24, 2017.

\bibitem[Alon et~al.(2017)Alon, Cesa-Bianchi, Gentile, Mannor, Mansour, and Shamir]{alon2017nonstochastic}
Noga Alon, Nicolo Cesa-Bianchi, Claudio Gentile, Shie Mannor, Yishay Mansour, and Ohad Shamir.
\newblock Nonstochastic multi-armed bandits with graph-structured feedback.
\newblock \emph{SIAM Journal on Computing}, 46\penalty0 (6):\penalty0 1785--1826, 2017.

\bibitem[Audibert \& Bubeck(2009)Audibert and Bubeck]{audibert2009minimax}
Jean-Yves Audibert and S{\'e}bastien Bubeck.
\newblock Minimax policies for adversarial and stochastic bandits.
\newblock In \emph{Conference on Learning Theory (COLT)}, volume~7, pp.\  1--122, 2009.

\bibitem[Auer et~al.(1995)Auer, Cesa-Bianchi, Freund, and Schapire]{auer1995gambling}
Peter Auer, Nicolo Cesa-Bianchi, Yoav Freund, and Robert~E. Schapire.
\newblock Gambling in a rigged casino: The adversarial multi-armed bandit problem.
\newblock In \emph{Proceedings of IEEE 36th annual foundations of computer science}, pp.\  322--331. IEEE, 1995.

\bibitem[Auer et~al.(2002)Auer, Cesa-Bianchi, and Fischer]{auer2002finite}
Peter Auer, Nicolo Cesa-Bianchi, and Paul Fischer.
\newblock Finite-time analysis of the multiarmed bandit problem.
\newblock \emph{Machine Learning}, 47\penalty0 (2):\penalty0 235--256, 2002.

\bibitem[Bartlett \& Wegkamp(2008)Bartlett and Wegkamp]{bartlett2008classification}
Peter~L Bartlett and Marten~H Wegkamp.
\newblock Classification with a reject option using a hinge loss.
\newblock \emph{Journal of Machine Learning Research}, 9\penalty0 (8), 2008.

\bibitem[Capp{\'e} et~al.(2013)Capp{\'e}, Garivier, Maillard, Munos, and Stoltz]{cappe2013kullback}
Olivier Capp{\'e}, Aur{\'e}lien Garivier, Odalric-Ambrym Maillard, R{\'e}mi Munos, and Gilles Stoltz.
\newblock {Kullback--Leibler} upper confidence bounds for optimal sequential allocation.
\newblock \emph{Annals of Statistics}, pp.\  1516--1541, 2013.

\bibitem[Cheng et~al.(2010)Cheng, Rademaker, De~Baets, and H{\"u}llermeier]{cheng2010predicting}
Weiwei Cheng, Micha{\"e}l Rademaker, Bernard De~Baets, and Eyke H{\"u}llermeier.
\newblock Predicting partial orders: {Ranking} with abstention.
\newblock In \emph{Machine Learning and Knowledge Discovery in Databases: European Conference, ECML PKDD 2010}, pp.\  215--230. Springer, 2010.

\bibitem[Chow(1957)]{chow1957optimum}
Chi-Keung Chow.
\newblock An optimum character recognition system using decision functions.
\newblock \emph{IRE Transactions on Electronic Computers}, pp.\  247--254, 1957.

\bibitem[Chow(1970)]{chow1970optimum}
Chi-Keung Chow.
\newblock On optimum recognition error and reject tradeoff.
\newblock \emph{IEEE Transactions on Information Theory}, 16\penalty0 (1):\penalty0 41--46, 1970.

\bibitem[Cortes et~al.(2016)Cortes, DeSalvo, and Mohri]{cortes2016learning}
Corinna Cortes, Giulia DeSalvo, and Mehryar Mohri.
\newblock Learning with rejection.
\newblock In \emph{Algorithmic Learning Theory: 27th International Conference, ALT 2016, Bari, Italy, October 19-21, 2016, Proceedings 27}, pp.\  67--82. Springer, 2016.

\bibitem[Cortes et~al.(2018)Cortes, DeSalvo, Gentile, Mohri, and Yang]{cortes2018online}
Corinna Cortes, Giulia DeSalvo, Claudio Gentile, Mehryar Mohri, and Scott Yang.
\newblock Online learning with abstention.
\newblock In \emph{International Conference on Machine Learning}, pp.\  1059--1067. PMLR, 2018.

\bibitem[Garivier et~al.(2019)Garivier, M{\'e}nard, and Stoltz]{garivier2019explore}
Aur{\'e}lien Garivier, Pierre M{\'e}nard, and Gilles Stoltz.
\newblock Explore first, exploit next: The true shape of regret in bandit problems.
\newblock \emph{Mathematics of Operations Research}, 44\penalty0 (2):\penalty0 377--399, 2019.

\bibitem[Herbei \& Wegkamp(2006)Herbei and Wegkamp]{herbei2006classification}
Radu Herbei and Marten~H Wegkamp.
\newblock Classification with reject option.
\newblock \emph{The Canadian Journal of Statistics / La Revue Canadienne de Statistique}, pp.\  709--721, 2006.

\bibitem[Honda \& Takemura(2010)Honda and Takemura]{honda2010asymptotically}
Junya Honda and Akimichi Takemura.
\newblock An asymptotically optimal bandit algorithm for bounded support models.
\newblock In \emph{Conference on Learning Theory}, pp.\  67--79. Citeseer, 2010.

\bibitem[Jin et~al.(2021)Jin, Xu, Shi, Xiao, and Gu]{jin2021mots}
Tianyuan Jin, Pan Xu, Jieming Shi, Xiaokui Xiao, and Quanquan Gu.
\newblock Mots: Minimax optimal thompson sampling.
\newblock In \emph{International Conference on Machine Learning}, pp.\  5074--5083. PMLR, 2021.

\bibitem[Jin et~al.(2023)Jin, Yang, Xiao, and Xu]{jin23b}
Tianyuan Jin, Xianglin Yang, Xiaokui Xiao, and Pan Xu.
\newblock Thompson sampling with less exploration is fast and optimal.
\newblock In \emph{Proceedings of the 40th International Conference on Machine Learning}, volume 202, pp.\  15239--15261. PMLR, 2023.

\bibitem[Kalai \& Kanade(2021)Kalai and Kanade]{kalai2021towards}
Adam Kalai and Varun Kanade.
\newblock Towards optimally abstaining from prediction with ood test examples.
\newblock \emph{Advances in Neural Information Processing Systems}, 34:\penalty0 12774--12785, 2021.

\bibitem[Kaufmann et~al.(2012)Kaufmann, Korda, and Munos]{kaufmann2012thompson}
Emilie Kaufmann, Nathaniel Korda, and R{\'e}mi Munos.
\newblock Thompson sampling: An asymptotically optimal finite-time analysis.
\newblock In \emph{Algorithmic Learning Theory: 23rd International Conference, ALT 2012, Lyon, France, October 29-31, 2012. Proceedings 23}, pp.\  199--213. Springer, 2012.

\bibitem[Korda et~al.(2013)Korda, Kaufmann, and Munos]{korda2013thompson}
Nathaniel Korda, Emilie Kaufmann, and R{\'e}mi Munos.
\newblock Thompson sampling for 1-dimensional exponential family bandits.
\newblock \emph{Advances in Neural Information Processing Systems}, 26, 2013.

\bibitem[Lai \& Robbins(1985)Lai and Robbins]{lai1985asymptotically}
Tze~Leung Lai and Herbert Robbins.
\newblock Asymptotically efficient adaptive allocation rules.
\newblock \emph{Advances in Applied Mathematics}, 6\penalty0 (1):\penalty0 4--22, 1985.

\bibitem[Lattimore(2018)]{lattimore2018refining}
Tor Lattimore.
\newblock Refining the confidence level for optimistic bandit strategies.
\newblock \emph{Journal of Machine Learning Research}, 19\penalty0 (1):\penalty0 765--796, 2018.

\bibitem[Lattimore \& Szepesv{\'a}ri(2017)Lattimore and Szepesv{\'a}ri]{lattimore2017end}
Tor Lattimore and Csaba Szepesv{\'a}ri.
\newblock The end of optimism? {An} asymptotic analysis of finite-armed linear bandits.
\newblock In \emph{Artificial Intelligence and Statistics}, pp.\  728--737. PMLR, 2017.

\bibitem[Lattimore \& Szepesv{\'a}ri(2020)Lattimore and Szepesv{\'a}ri]{lattimore2020bandit}
Tor Lattimore and Csaba Szepesv{\'a}ri.
\newblock \emph{Bandit Algorithms}.
\newblock Cambridge University Press, 2020.

\bibitem[Littlestone \& Warmuth(1994)Littlestone and Warmuth]{littlestone1994weighted}
Nick Littlestone and Manfred~K Warmuth.
\newblock The weighted majority algorithm.
\newblock \emph{Information and Computation}, 108\penalty0 (2):\penalty0 212--261, 1994.

\bibitem[Liu et~al.(2018)Liu, Buccapatnam, and Shroff]{liu2018information}
Fang Liu, Swapna Buccapatnam, and Ness Shroff.
\newblock Information directed sampling for stochastic bandits with graph feedback.
\newblock In \emph{Proceedings of the AAAI Conference on Artificial Intelligence}, volume~32, 2018.

\bibitem[Mao et~al.(2023)Mao, Mohri, and Zhong]{mao2023ranking}
Anqi Mao, Mehryar Mohri, and Yutao Zhong.
\newblock Ranking with abstention.
\newblock \emph{arXiv preprint arXiv:2307.02035}, 2023.

\bibitem[M{\'e}nard \& Garivier(2017)M{\'e}nard and Garivier]{menard2017minimax}
Pierre M{\'e}nard and Aur{\'e}lien Garivier.
\newblock A minimax and asymptotically optimal algorithm for stochastic bandits.
\newblock In \emph{International Conference on Algorithmic Learning Theory}, pp.\  223--237. PMLR, 2017.

\bibitem[Neu \& Zhivotovskiy(2020)Neu and Zhivotovskiy]{neu2020fast}
Gergely Neu and Nikita Zhivotovskiy.
\newblock Fast rates for online prediction with abstention.
\newblock In \emph{Conference on Learning Theory}, pp.\  3030--3048. PMLR, 2020.

\bibitem[Thompson(1933)]{thompson1933likelihood}
William~R. Thompson.
\newblock On the likelihood that one unknown probability exceeds another in view of the evidence of two samples.
\newblock \emph{Biometrika}, 25\penalty0 (3-4):\penalty0 285--294, 1933.

\bibitem[Tsybakov(2009)]{Tsybakov:1315296}
Alexandre~B. Tsybakov.
\newblock \emph{{Introduction to Nonparametric Estimation}}.
\newblock Springer Series in Statistics. Springer, 2009.

\bibitem[Wiener \& El-Yaniv(2012)Wiener and El-Yaniv]{wiener2012pointwise}
Yair Wiener and Ran El-Yaniv.
\newblock Pointwise tracking the optimal regression function.
\newblock \emph{Advances in Neural Information Processing Systems}, 25, 2012.

\bibitem[Zaoui et~al.(2020)Zaoui, Denis, and Hebiri]{zaoui2020regression}
Ahmed Zaoui, Christophe Denis, and Mohamed Hebiri.
\newblock Regression with reject option and application to knn.
\newblock \emph{Advances in Neural Information Processing Systems}, 33:\penalty0 20073--20082, 2020.

\end{thebibliography}
\bibliographystyle{tmlr}

\appendix

\section{Asymptotic and Minimax Optimality in Canonical Multi-Armed Bandits }
\label{appendix_standard_bandits}

In the canonical multi-armed bandit model, there is no additional abstention option. Given a bandit instance $\mu\in\mathcal U$, at each time step $t\in\mathbb N$, the agent employs an {\em online algorithm} $\pi$ to choose an arm $A_t$ from the arm set $[K]$, and then observes a random variable $X_t$ from the selected arm $A_t$, which is drawn from a Gaussian distribution $\mathcal N (\mu_{A_t}, 1)$ and independent of  observations from previous time steps. The choice of  $A_t$ might depend on the prior decisions and observations. To describe the setup formally, $A_t$ is $\mathcal F^{\CA}_{t-1}$-measurable, where $\mathcal F^{\CA}_t :=\sigma(A_1,X_1,A_2,X_2,\ldots,A_t, X_t)$ represents the  $\sigma$-field generated by the cumulative interaction history up to and including time $t$. Subsequently, the agent suffers an instantaneous regret of $\mu_1-X_t$.

The agent aims at minimizing the expected cumulative regret over a time horizon $T$, which is defined as 
$$
R^{\CA}_{\mu}(T, \pi) = T\mu_1 - \mathbb E\left[\sum_{t=1}^T X_t \right].
$$

We refer to the collection of all online policies for the canonical multi-armed bandit model as $\Pi^{\CA}$.

\begin{remark}
    It is worth noting that any algorithm designed for canonical multi-armed bandit model possesses the capability to decide the arm $A_t$ to pull at each time step $t$, based on the partial interaction history $(A_1,X_1,A_2,X_2, \ldots,A_{t-1}, X_{t-1})$, within the abstention model. Conversely, any algorithm tailored for the abstention model in the fixed-regret setting (or in the fixed-reward setting) can be applied to the canonical multi-armed bandit model, provided that the abstention regret (or the abstention reward) has been predetermined. Specifically, the algorithm can determine both the selected arm $A_t$ and the binary abstention variable $B_t$, although $B_t$ is purely auxiliary and exerts no influence on the cumulative regret $R^{\CA}_{\mu}(T, \pi)$.
\end{remark}

\paragraph{Lower bounds.} Both the asymptotic and minimax lower bounds for the canonical multi-armed bandit model have been thoroughly established \citep{lai1985asymptotically,auer1995gambling}. For a comprehensive overview, refer to Sections~15 and 16 of \citet{lattimore2020bandit}. Here, we summarize the results in the following:

\begin{definition}[$R^{\CA}$-consistency]
An algorithm $\pi\in \Pi^{\CA}$ is said to be \emph{$R^{\CA}$-consistent} if for all bandit instances $\mu\in \mathcal U$ and $a>0$, $R^{\CA}_{\mu}(T, \pi)=o(T^a)$.
\label{definition_standard_consistent}
\end{definition}
\begin{theorem} 
For any bandit instance $\mu\in \mathcal U$ and  $R^{\CA}$-consistent algorithm~$\pi$, it holds that
$$
\liminf_{T\to \infty} \frac {R^{\CA}_{\mu}(T, \pi)} {\log T} \ge  \sum_{i>1} \frac 2 {\Delta_i}.
$$

For any time horizon $T\ge K$, there exists a universal constant $\alpha > 0$ such that
$$
\inf_{\pi\in \Pi^{\CA}} \sup _{\mu\in \mathcal U} {R^{\CA}_{\mu}(T, \pi)}  \ge \alpha \sqrt{KT}.
$$
\label{theorem_standard_lowerbound}
\end{theorem}

\paragraph{Asymptotic and minimax optimality.} According to Theorem~\ref{theorem_standard_lowerbound}, in the canonical bandit model, an algorithm $\pi\in \Pi^{\CA} $ is said to be \emph{asymptotically optimal} if for all bandit instances $\mu\in\mathcal U$, it ensures that 
$$
\lim_{T\to \infty} \frac {R^{\CA}_{\mu}(T, \pi)} {\log T} =  \sum_{i>1} \frac 2 {\Delta_i}.
$$
Furthermore, it is said to be \emph{minimax optimal} if there exists a universal constant $\alpha > 0$ such that
$$
{R^{\CA}_{\mu}(T, \pi)}  \le \alpha \left(\sqrt{KT}+\sum_{i>1} \Delta_i \right).
$$

To the best of our knowledge, for canonical multi-armed bandits with Gaussian rewards, $\text{KL-UCB}^{++}$ \citep{menard2017minimax}, ADA-UCB \citep{lattimore2018refining}, MOTS-$\mathcal{J}$ \citep{jin2021mots} and Less-Exploring Thompson Sampling \citep{jin23b} exhibit simultaneous asymptotic and minimax optimality. As their names suggest, the former two algorithms follow the UCB-style, while the latter two are rooted in Thompson Sampling.

\begin{remark}
\label{remark_optimal_property}
One valuable byproduct derived from the proof of the asymptotic lower bound in Theorem~\ref{theorem_standard_lowerbound} is that, for any  $R^{\CA}$-consistent algorithm~$\pi$, bandit instance $\mu\in \mathcal U$ and suboptimal arm $i>1$, we have 
 $$
 \liminf _{T \rightarrow \infty} \frac{\E[N_i(T)]}{\log T } \geq \frac{2}{\Delta_i^2} .
 $$
Therefore, any algorithm that is asymptotically optimal ensures that for all suboptimal arms $i>1$, 
 $$
 \lim _{T \rightarrow \infty} \frac{\E[N_i(T)]}{\log T } = \frac{2}{\Delta_i^2} .
 $$
\end{remark}

\section{Comparison with the \texorpdfstring{$(K+1)$}{(K+1)}-Armed Bandit Reduction}
\label{appendix_comparison_k_plus_1}

A natural question arises regarding the formulation of our abstention model: can the problem be simply reduced to a standard multi-armed bandit setting with $K+1$ arms? In such a reduction, the additional $(K+1)$-st arm would represent the abstention option, yielding a deterministic reward $c$ (in the fixed-reward setting) or a fixed regret (in the fixed-regret setting).

While such a reduction might appear intuitive and indeed may suffice for achieving coarse worst-case (minimax) guarantees of $O(\sqrt{KT})$, it fundamentally fails to capture the strategic advantages of our framework. Specifically, the reduction cannot achieve the same instance-dependent asymptotic optimality. In the following, we discuss the essential differences between our model and the $(K+1)$-armed reduction.

\paragraph{Structural Difference: Decoupling Learning from Reward.}
The most critical distinction lies in the mechanism of information acquisition. In a standard $(K+1)$-armed bandit model, selecting the deterministic $(K+1)$-th arm (the ``safe'' option) fixes the instantaneous outcome but reveals \emph{no information} regarding the distributions of the original $K$ stochastic arms. To learn about a specific arm $i \in [K]$, the agent is forced to pull it and fully accept its stochastic outcome.

In contrast, our abstention framework allows the agent to  \emph{decouple} the observation of an outcome from the acceptance of its reward. When the agent exercises the abstention option, they secure a guaranteed result (equivalent to the deterministic arm) but, crucially, \emph{simultaneously} observe the stochastic sample $X_t$ from the selected arm $A_t$. This structure enables what can be termed ``safe exploration'': the agent can continue to gather valuable statistics about potentially risky or unknown arms to refine their estimates, all while shielding themselves from the full cost of stochastic underperformance.

\paragraph{Performance Gap: Asymptotic Optimality.}
To illustrate how this structural difference translates into superior asymptotic performance, we show below that in both the fixed-reward and fixed-regret objectives, the $(K+1)$-armed reduction and our abstention model admit different instance-dependent asymptotic behavior.

\textit{Fixed Reward: }In a standard $(K+1)$-armed bandit reduction, to learn about a suboptimal arm $i$, the agent must select it. Consequently, the agent incurs a regret proportional to the full suboptimality gap $\Delta_i = \mu_1 - \mu_i$ for every exploratory sample. The asymptotic regret contribution from arm $i$ is thus driven by the factor $\Delta_i$.
However, in our fixed-reward abstention model, the agent can leverage the abstention option to mitigate this cost. Specifically, if the abstention reward $c$ satisfies $\mu_i < c < \mu_1$, the agent can pull arm $i$ while simultaneously choosing to abstain. In this scenario, the agent observes the outcome of arm $i$ (gaining information) but receives the fixed reward $c$. The resulting instantaneous regret is $\mu_1 - c$, which is strictly less than the full gap $\mu_1 - \mu_i$.

This mechanism facilitates a more cost-effective form of exploration. Consequently, our derived asymptotic regret bound scales with the term $\mu_1 \vee c - \mu_i \vee c$, rather than the larger gap $\Delta_i$ that arises from the (K+1)-armed reduction.  This confirms that a simple reduction to a $(K+1)$-armed bandit cannot replicate the superior asymptotic optimality achieved by our proposed model.

\textit{Fixed Regret: }A simple $(K+1)$-armed reduction also fails in the fixed-regret objective.
In that reduction, the $(K+1)$-st ``abstain'' arm incurs a deterministic regret $c$,
but it reveals \emph{no information} about the original $K$ stochastic arms.
Hence, to statistically verify that a suboptimal arm $i$ is worse than the best arm,
the learner must still pull arm $i$ and \emph{accept} its stochastic outcome, paying regret about $\Delta_i$ per  sample.
By the standard asymptotic lower bound in \citet{lai1985asymptotically}, the regret caused by pulling arm $i$ is on the order of $\frac{2\log T}{\Delta_i}$.
Therefore, the leading asymptotic term under the $(K+1)$-arm reduction is
$\sum_{i>1}\frac{2\log T}{\Delta_i}$,
which is \emph{independent of $c$}.
In contrast, in our abstention model the learner can pull arm $i$ while abstaining, still observing $X_t$ but capping the instantaneous regret by $c$.
Thus, exploration of arm $i$ costs $\Delta_i\wedge c$ per sample, yielding the strictly smaller scaling
$2\sum_{i>1}\frac{\Delta_i\wedge c}{\Delta_i^2}\log T$,
and in particular improves $\frac{2\log T}{\Delta_i}$ to $\frac{2c\log T}{\Delta_i^2}$ whenever $\Delta_i>c$.
This again shows the $(K\!+\!1)$-armed reduction cannot replicate the fixed-regret asymptotic achieved by our model.

\color{black}

\section{Auxiliary Lemmas}
\begin{lemma}[Bretagnolle--Huber inequality \citep{Tsybakov:1315296}] Let $\mathbb P$ and $\mathbb P'$ be two probability distributions on the same measurable space $(\Omega, \mathcal{F})$. For any event $A \in \mathcal{F}$ and its complement $A^c = \Omega \setminus A$, the following inequality holds:
$$\mathbb P(A)+\mathbb P' (A^c ) \ge \frac{1}{2} \exp (-\mathrm{KL}(\mathbb P, \mathbb P')),$$
where $\mathrm{KL}(\mathbb P, \mathbb P')$ denotes the Kullback--Leibler (KL) divergence between $\mathbb P$ and $\mathbb P'$.
\label{lemma_BH}
\end{lemma}

\vspace{0pt}

\begin{lemma}[Divergence decomposition lemma]
\label{lemma_divergence_decomposition}
Consider both the fixed-regret setting and the fixed-reward setting.  Fix an arbitrary policy $\pi$. Let $\nu= (\mathbb{P}_1, \ldots, \mathbb{P}_K)$ represent the reward distributions associated with one bandit instance, and let $\nu'=(\mathbb{P}_1', \ldots, \mathbb{P}_K')$ represent the reward distributions associated with another bandit instance.  Define $\mathbb P_{\nu,c}$  as the probability distribution of the sequence $(A_1,B_1,X_1,\ldots,A_T,B_T, X_T)$ induced by the algorithm $\pi$  under the abstention regret $c$ in the fixed-regret setting (or the abstention reward $c$ in the fixed-reward setting) for the bandit instance~$\nu$. Similarly, let $\mathbb{P}_{\nu',c}$ denote the same for the bandit instance $\nu'$. Then the KL divergence between $\mathbb P_{\nu,c}$ and $\mathbb P_{\nu',c}$ can be decomposed as:
$$
\mathrm{KL}\left(\mathbb P_{\nu,c}, \mathbb P_{\nu',c}\right) = \sum_{i\in[K]} \E_{\nu,c} [N_i(T)] \mathrm{KL}\left(\mathbb P_i, \mathbb P_i'\right).
$$
\end{lemma}

The proof of Lemma~\ref{lemma_divergence_decomposition} is similar to the well-known proof of divergence decomposition in the canonical multi-armed bandit model (excluding abstention), and is therefore omitted. This proof can be located, for instance, in \citet[Section 2.1]{garivier2019explore} and \citet[Lemma 15.1]{lattimore2020bandit}.

\begin{lemma}[Hoeffding's inequality for sub-Gaussian random variables]
\label{lemma_hoeffding1}
Let $X_1, \ldots, X_n$ be independent $\sigma$-sub-Gaussian random variables with mean $\mu$. Then for any $\varepsilon \ge 0$,
$$
\mathbb{P}(\hat{\mu} \geq \mu+\varepsilon) \leq \exp \left(-\frac{n \varepsilon^2}{2 \sigma^2}\right) \quad \text { and } \quad \mathbb{P}(\hat{\mu} \leq \mu-\varepsilon) \leq \exp \left(-\frac{n \varepsilon^2}{2 \sigma^2}\right)
$$
where $\hat{\mu} :=\frac{1}{n} \sum_{i=1}^n X_i$. 
\end{lemma}

\vspace{0pt}

\begin{lemma}
\label{lemma_hoeffding2}
Let $\{X_i\}_{i\in \mathbb{N}}$ be a sequence of independent $\sigma$-sub-Gaussian random variables with mean $\mu$. Then for any $\varepsilon > 0$ and $N\in \mathbb{N}$, 
$$
\sum_{n=1}^N \mathbb{P}(\hat{\mu}_n \geq \mu+\varepsilon) \leq \frac{2 \sigma^2} {\varepsilon^2}  \quad \text { and } \quad \sum_{n=1}^N \mathbb{P}(\hat{\mu}_n \leq \mu-\varepsilon) \leq \frac{2 \sigma^2} {\varepsilon^2}
$$
where $\hat{\mu}_n :=\frac{1}{n} \sum_{i=1}^n X_i$. 
\end{lemma}

\begin{proof}
By symmetry, it suffices to prove the first part. According to Lemma~\ref{lemma_hoeffding1}, we have
\begin{align*}
    \sum_{n=1}^N \mathbb{P}(\hat{\mu}_n \geq \mu+\varepsilon) &\leq  \sum_{n=1}^N \exp \left(-\frac{n \varepsilon^2}{2 \sigma^2}\right) \\
    &\le \frac{\exp \left(-\frac{\varepsilon^2}{2 \sigma^2}\right)} {1- \exp \left(-\frac{ \varepsilon^2}{2 \sigma^2}\right)} \\
    &= \frac {1} {\exp \left(\frac{ \varepsilon^2}{2 \sigma^2}\right) -1 } \\
    &\le \frac{2 \sigma^2} {\varepsilon^2}
\end{align*}
where the last inequality follows from the fact that $e^x-1\ge x$ for any $x\ge 0$.
\end{proof}

\section{Analysis of the Fixed-Regret Setting}

\subsection{Upper Bounds}
\label{subsection_fixedregret_upperboundproof}
\begin{proof}[Proof of Theorem~\ref{theorem_fixedregret_upperbound}]
Due to the law of total expectation, we can decompose the regret $R^{\FRG}_{\mu,c}(T, \pi)$ as 
\begin{align}
R^{\FRG}_{\mu,c}(T, \pi) &=\mathbb E\left[\sum_{t=1}^T\big ((\mu_1-X_t)\cdot\mathbf 1\{B_t=0\}+c\cdot\mathbf 1\{B_t=1\}\big )\right] \notag \\
&=\mathbb E\left[\sum_{t=1}^T\big ((\mu_1-\mu_{A_t})\cdot\mathbf 1\{B_t=0\}+c\cdot\mathbf 1\{B_t=1\}\big )\right] \notag \\
&= c\cdot \E[N_1^{(1)}(T)] +\sum_{i>1}  \left( \Delta_i \cdot \E[N_i^{(0)}(T)]  + c\cdot \E[N_i^{(1)}(T)]\right). \label{equation_fixedregret_regretdecompose}
\end{align}

For any arm $i$ with $\Delta_i <c$ (including the best arm), it holds that
\begin{align}
&\phantom{\ = \ } \E[N_i^{(1)}(T)] \notag  \\ &=\E\left[\sum_{t=1}^T \mathbb{I}\left\{ A_t=i \text { and } B_t=1 \right\}\right] \notag \\
&\le \E\left[\sum_{t=1}^T \mathbb{I}\left\{ A_t=i \text { and } \sqrt{\frac K t} \ge c \right\}\right]  \notag \\
&\phantom{\ = \ } + \E\left[\sum_{t=K+1}^T \mathbb{I}\left\{ A_t=i \text { and } \max_{j\in [K]\setminus \{i\}} \left(\hat{\mu}_{j}(t-1)-\sqrt{\frac{6\log t+2\log(c\vee 1)}{N_{j}(t-1)}}\right) - \hat{\mu}_{i}(t-1) \ge c \right\}\right] \notag \\
&\le \E\left[\sum_{t=1}^T \mathbb{I}\left\{ A_t=i \text { and } \sqrt{\frac K t} \ge c  \right\}\right]    \notag \\
&\phantom{\ = \ } +\E\left[\sum_{t=K+1}^T  \mathbb{I}\left\{  A_t=i \text { and } \max_{j\in [K]} \left(\hat{\mu}_{j}(t-1)-\sqrt{\frac{6\log t+2\log(c\vee 1)}{N_{j}(t-1)}}\right) \ge \mu_1 \right\}\right]  \notag \\
&\phantom{\ = \ } +\E\left[\sum_{t=K+1}^T \mathbb{I}\left\{ A_t=i \text { and }  \mu_1-\hat{\mu}_i(t-1) \ge c \right\}\right] \label{equation_case1_three terms} 
\end{align}
where the last inequality arises from the observation that when $A_t = i$, 
\begin{align*}
    &\phantom{\ = \ }\left\{ \max_{j\in [K]\setminus \{i\}} \left(\hat{\mu}_{j}(t-1)-\sqrt{\frac{6\log t+2\log(c\vee 1)}{N_{j}(t-1)}}\right) - \hat{\mu}_{i}(t-1) \ge c \right\} \\
    &\subseteq \left\{\max_{j\in [K]\setminus \{i\}} \left(\hat{\mu}_{j}(t-1)-\sqrt{\frac{6\log t+2\log(c\vee 1)}{N_{j}(t-1)}}\right) \ge \mu_1 \right\}  \cup \left\{   \mu_1-\hat{\mu}_i(t-1) \ge c \right\} \\
    &\subseteq \left\{\max_{j\in [K]} \left(\hat{\mu}_{j}(t-1)-\sqrt{\frac{6\log t+2\log(c\vee 1)}{N_{j}(t-1)}}\right) \ge \mu_1 \right\} \cup \left\{  \mu_1-\hat{\mu}_i(t-1) \ge c \right\}.
\end{align*}

For convenience, for any $i\in[K]$ such that $\Delta_i <c$, we introduce three shorthand notations to represent the terms in Equation~\eqref{equation_case1_three terms}. These terms correspond to distinct conditions under which the algorithm chooses to abstain:
$$
\begin{cases}
    \vspace*{4pt}(\clubsuit)_i := \E\left[\sum_{t=1}^T \mathbb{I}\left\{ A_t=i \text { and } \sqrt{\frac K t} \ge c  \right\}\right] \\
    \vspace*{4pt}
    (\spadesuit)_i := \E\left[\sum_{t=K+1}^T  \mathbb{I}\left\{ A_t=i \text { and } \max_{j\in [K]} \left(\hat{\mu}_{j}(t-1)-\sqrt{\frac{6\log t+2\log(c\vee 1)}{N_{j}(t-1)}}\right) \ge \mu_1 \right\}\right] \\
    (\blacksquare)_i := \E\left[\sum_{t=K+1}^T \mathbb{I}\left\{ A_t=i \text { and }  \mu_1-\hat{\mu}_i(t-1) \ge c \right\}\right] .
\end{cases}
$$

To clarify the intuition behind these terms:
\begin{itemize}
    \item $(\clubsuit)_i$ represents the gap-independent abstention. This captures the case where the algorithm abstains purely because the time step $t$ is small, independent of the empirical estimates.
    \item $(\spadesuit)_i$ represents abstentions caused by the {overestimation of LCBs}. This captures the low-probability event where the lower confidence bound of any arm $j$ exceeds the true mean of the optimal arm $\mu_1$.
    \item $(\blacksquare)_i$ represents abstentions caused by the {underestimation of the selected arm}. This captures the event where the empirical mean of the chosen arm $i$ deviates significantly below its true mean $\mu_i$ (specifically such that $\mu_1 - \hat{\mu}_i(t-1) \ge c$), triggering the gap-dependent abstention criterion.
\end{itemize}

We will analyze $(\clubsuit)_i $ and $(\spadesuit)_i$ later to establish the two forms of upper bounds.

On the other hand, for the term $(\blacksquare)_i$, which relates strictly to the concentration of arm $i$, we have 
\color{black}
\begin{align}
    (\blacksquare)_i &\le  \E\left[\sum_{t=K+1}^T \sum_{s=1}^{T-1} \mathbb{I}\left\{ A_t=i \text { and } \hat{\mu}_{is} \le \mu_i  - (c-\Delta_i ) \text { and }  N_i(t-1)=s \right\}\right] \notag \\
    &\le \E\left[\sum_{s=1}^{T-1} \mathbb{I}\left\{ 
  \hat{\mu}_{is}\le \mu_i  - (c-\Delta_i ) \right\}\right ]  \label{equation_blacksquare1} \\
  &\le  \frac 2 {(c-\Delta_i)^2}. \label{equation_blacksquare2}
\end{align}

Line~\eqref{equation_blacksquare1} follows from the fact that for all $s\in[T-1]$, 
$$
\sum_{t=K+1}^T \mathbb{I} \left\{ A_t=i \text { and } N_i(t-1)=s \right\} \le 1 .
$$

Line~\eqref{equation_blacksquare2} is due to Lemma~\ref{lemma_hoeffding2}.

For any arm $i$ with $\Delta_i >c$, since arm $1\in [K]\setminus \{i\}$, we have
\begin{align}
&\phantom{\ = \ }\E[N_i^{(0)}(T)]  \notag\\
&=\E\left[\sum_{t=1}^T \mathbb{I}\left\{ A_t=i \text { and } B_t=0 \right\}\right] \notag\\
&\le 1 + \E\left[\sum_{t=K+1}^T \mathbb{I}\left\{ A_t=i \text { and } \phantom {\max_{j\in [K]\setminus \{i\}} \left(\hat{\mu}_{j}(t-1)-\sqrt{\frac{6\log t+2\log(c\vee 1)}{N_{j}(t-1)}}\right) } \right. \right. \notag\\ 
& \phantom{1111111111}\left.\left.\max_{j\in [K]\setminus \{i\}} \left(\hat{\mu}_{j}(t-1)-\sqrt{\frac{6\log t+2\log(c\vee 1)}{N_{j}(t-1)}}\right) - \hat{\mu}_{i}(t-1) < c  \text { and } \sqrt{\frac K t} < c \right\}\right] \notag \\
&\le 1 + \E\left[\sum_{t=K+1}^T \mathbb{I}\left\{ A_t=i \text { and } \left(\hat{\mu}_{1}(t-1)-\sqrt{\frac{6\log t+2\log(c\vee 1)}{N_{1}(t-1)}}\right) - \hat{\mu}_{i}(t-1) < c 
\right\}\right]. \label{equation_case2}
\end{align}

\paragraph{Minimax upper bound. } If $c\le \sqrt{\frac K  T}$, then the abstention option is always invoked because $\sqrt{\frac K  t} \ge \sqrt{ \frac K  T}\ge c$ for all $t\in[T]$. Consequently, it is straightforward to deduce that
$$
{R^{\FRG}_{\mu,c}(T)} \le cT.
$$

Next, consider the case that $c > \sqrt{\frac K  T}$. Compared with the canonical multi-armed bandit model, at a single time step, the agent in our abstention model incurs a greater (expected) regret only if an arm $i$ with $\Delta_i <c$ is pulled and the abstention option is selected.  Thus, we have 
\begin{equation}
{R^{\FRG}_{\mu,c}(T)} \le R^{\CA}_{\mu}(T) + \sum_{i:\Delta_i <c} (c-\Delta_i)  \cdot \E[N_i^{(1)}(T)]. 
\label{equation_fixedregret_minimax0}
\end{equation}

Due to the minimax optimality of Less-Exploring Thompson Sampling \citep{jin23b}, there exists a universal constant $\alpha_1 > 0$ such that
\begin{equation}
R^{\CA}_{\mu}(T) \le \alpha_1 \left(\sqrt{KT}+\sum_{i>1}{\Delta_i}\right).
\label{equation_fixedregret_minimax1}
\end{equation}

Recall the upper bound of $\E[N_i^{(1)}(T)]$ for arm $i$ with $\Delta_i <c$, as given in~\eqref{equation_case1_three terms}. Subsequently, we will establish bounds for the following terms:
$$\sum_{i:\Delta_i <c} (c-\Delta_i)  \cdot (\clubsuit)_i \ , \quad \sum_{i:\Delta_i <c} (c-\Delta_i)  \cdot (\spadesuit)_i \quad \text{and} \quad \sum_{i:\Delta_i <c} (c-\Delta_i)  \cdot ((\blacksquare)_i.$$ 

For the first term, we have
\begin{align}
    \sum_{i:\Delta_i <c} (c-\Delta_i)  \cdot (\clubsuit)_i &\le   \sum_{i:\Delta_i <c} c  \cdot \E\left[\sum_{t=1}^T \mathbb{I}\left\{ A_t=i \text { and } \sqrt{\frac K t} \ge c  \right\}\right] \notag\\
    &\le c\cdot \E\left[\sum_{t=1}^T \mathbb{I}\left\{ \sqrt{\frac K t} \ge c  \right\}\right] \notag\\
    &\le \frac K c \notag \\
    &\le \sqrt{KT}. \label{equation_fixedregret_minimax2}
\end{align}

For the second term, we can obtain 
\begin{align*}
    \sum_{i:\Delta_i <c} (c-\Delta_i)  \cdot (\spadesuit)_i &\le  \sum_{i:\Delta_i <c} c  \cdot (\spadesuit)_i \\
    &\le  c  \cdot \E\left[\sum_{t=K+1}^T  \mathbb{I}\left\{ \max_{j\in [K]} \left(\hat{\mu}_{j}(t-1)-\sqrt{\frac{6\log t+2\log(c\vee 1)}{N_{j}(t-1)}}\right) \ge \mu_1 \right\}\right]  \\
    &\le c\cdot \sum_{j\in[K]}\E\left[\sum_{t=K+1}^T  \mathbb{I}\left\{ \hat{\mu}_{j}(t-1)-\sqrt{\frac{6\log t+2\log(c\vee 1)}{N_{j}(t-1)}} \ge \mu_1 \right\}\right].
\end{align*}

For all $j\in[K]$, by a union bound over all possible values of $N_{j}(t-1)$ and Lemma~\ref{lemma_hoeffding1}, we have 
\begin{equation}
\begin{aligned}
&\phantom{\ = \ }\E\left[\sum_{t=K+1}^T  \mathbb{I}\left\{ \hat{\mu}_{j}(t-1)-\sqrt{\frac{6\log t+2\log(c\vee 1)}{N_{j}(t-1)}} \ge \mu_1 \right\}\right] \\
&\le \sum_{t=K+1}^T \sum_{s=1}^{t-1} \mathbb P \left (\hat{\mu}_{js}-\sqrt{\frac{6\log t+2\log(c\vee 1)}{s}} \ge \mu_1 \right) \\
&\le \sum_{t=K+1}^T \sum_{s=1}^{t-1} \mathbb P \left (\hat{\mu}_{js}-\sqrt{\frac{6\log t+2\log(c\vee 1)}{s}} \ge \mu_j \right) \\
&\le \sum_{t=K+1}^T \sum_{s=1}^{t-1} \frac 1 {t^3(c\vee 1)} \\
&\le \sum_{t=K+1}^T  \frac 1 {t^2c} \\
&\le \frac 1 {Kc}
\end{aligned}
\label{equation_fixedregret_result1}
\end{equation}
where the last inequality follows from the numerical fact that
$$
\sum_{t=K+1}^T  \frac 1 {t^2} \le \int_{x=K}^{\infty} \frac 1 {x^2}\, \mathrm{d} x = \frac 1 K.
$$

Thus, we can bound the second term as 
\begin{align}
    \sum_{i:\Delta_i <c} (c-\Delta_i)  \cdot (\spadesuit)_i &\le c\cdot \sum_{j\in[K]}  \frac 1 {Kc} = 1 .
\label{equation_fixedregret_minimax3}
\end{align}

For the third term, in addition to the upper bound of $(\blacksquare)_i$ in \eqref{equation_blacksquare2}, we can identify another straightforward upper bound as follows:
$$
(\blacksquare)_i \le \E\left[\sum_{t=K+1}^T \mathbb{I}\left\{ A_t=i \right\}\right]  \le \E[N_i(T)].
$$

By applying these two bounds separately for distinct scenarios,, we have
\begin{align}
    \sum_{i:\Delta_i <c} (c-\Delta_i)  \cdot (\blacksquare)_i &\le \sum_{i:0<c-\Delta_i <\sqrt{\frac K T}}  (c-\Delta_i )\cdot (\blacksquare)_i + \sum_{i:c-\Delta_i\ge \sqrt{\frac K T}}  (c-\Delta_i )\cdot (\blacksquare)_i \notag \\
    &\le \sqrt{\frac K T} \sum_{i:0<c-\Delta_i <\sqrt{\frac K T}} \E[N_i(T)]  + \sum_{i:c-\Delta_i\ge \sqrt{\frac K T}}   \frac 2 {c-\Delta_i} \notag\\
    &\le \sqrt{\frac K T} \cdot T +\sum_{i:c-\Delta_i\ge \sqrt{\frac K T}}   2 \sqrt{\frac T K} \notag\\
    &\le \sqrt{KT} + K\cdot   2 \sqrt{\frac T K} \notag\\
    &= 3\sqrt{KT}. \label{equation_fixedregret_minimax4}
\end{align}

By plugging Inequalities~\eqref{equation_fixedregret_minimax2}, \eqref{equation_fixedregret_minimax3} and \eqref{equation_fixedregret_minimax4} into \eqref{equation_case1_three terms}, we have 
\begin{align*}
    \sum_{i:\Delta_i <c} (c-\Delta_i)  \cdot \E[N_i^{(1)}(T)] \le 1 + 4\sqrt{KT}.
\end{align*}

Together with \eqref{equation_fixedregret_minimax0} and \eqref{equation_fixedregret_minimax1}, we can conclude that 
\begin{equation*}
{R^{\FRG}_{\mu,c}(T)} \le (\alpha_1 + 4) \sqrt{KT}+\alpha_1 \sum_{i>1}{\Delta_i} + 1.
\end{equation*}

Therefore, there must exist a universal constant $\alpha > 0$ such that
$$
{R^{\FRG}_{\mu,c}(T)}  \le  \alpha  \left(\sqrt{KT}+\sum_{i>1}{\Delta_i}\right).
$$

This completes the proof of the minimax upper bound.

\paragraph{Asymptotic upper bound. } Consider any arm $i$ with $\Delta_i <c$ (including the best arm). In the following, we will further elucidate the upper bound of $\E[N_i^{(1)}(T)]$ as given in~\eqref{equation_case1_three terms} within the asymptotic domain.

For the first term $(\clubsuit)_i$ in~\eqref{equation_case1_three terms}, we have 
$$
(\clubsuit)_i \le \E\left[\sum_{t=1}^T \mathbb{I}\left\{\sqrt{\frac K t} \ge c  \right\}\right] \le \frac K {c^2}.
$$

For the second term $(\spadesuit)_i$, using Inequality~\eqref{equation_fixedregret_result1}, we can get
\begin{align*}
  (\spadesuit)_i &\le  \E\left[\sum_{t=K+1}^T  \mathbb{I}\left\{  \max_{j\in [K]} \left(\hat{\mu}_{j}(t-1)-\sqrt{\frac{6\log t+2\log(c\vee 1)}{N_{j}(t-1)}}\right) \ge \mu_1 \right\}\right]  \\
  &\le \sum_{j\in[K]}\E\left[\sum_{t=K+1}^T  \mathbb{I}\left\{ \hat{\mu}_{j}(t-1)-\sqrt{\frac{6\log t+2\log(c\vee 1)}{N_{j}(t-1)}} \ge \mu_1 \right\}\right] \\
  &\le \frac 1 c.
\end{align*}

Incorporating \eqref{equation_blacksquare2}, we obtain 
\begin{align*}
\E[N_i^{(1)}(T)] \le \frac K {c^2} + \frac 1 c + \frac 2 {(c-\Delta_i)^2} = o(\log T).
\end{align*}

Consider any arm $i$ with $\Delta_i >c$. We will further explore the  upper bound of $\E[N_i^{(0)}(T)]$ in~\eqref{equation_case2}.

According to the fact that 
\begin{align*}
    &\phantom{\ = \ }\left\{ A_t=i \text { and } \left(\hat{\mu}_{1}(t-1)-\sqrt{\frac{6\log t+2\log(c\vee 1)}{N_{1}(t-1)}}\right) - \hat{\mu}_{i}(t-1) < c \right\} \\
    &\subseteq \left\{ \hat{\mu}_{1}(t-1)+\sqrt{\frac{6\log t+2\log(c\vee 1)}{N_{1}(t-1)}} \le \mu_1 \right\}  \\
    &\phantom{\ = \ } \cup \left\{ A_t=i \text { and } \left({\mu}_{1}-2\sqrt{\frac{6\log t+2\log(c\vee 1)}{N_{1}(t-1)}}\right) - \hat{\mu}_{i}(t-1) \le  c \right\},
\end{align*}
we have 
\begin{align}
&\phantom{\ = \ }\E[N_i^{(0)}(T)] \notag \\
&\le 1 + \E\left[\sum_{t=K+1}^T \mathbb{I}\left\{ A_t=i \text { and } \left(\hat{\mu}_{1}(t-1)-\sqrt{\frac{6\log t+2\log(c\vee 1)}{N_{1}(t-1)}}\right) - \hat{\mu}_{i}(t-1) < c 
\right\}\right] \notag \\
&\le 1 + \E\left[\sum_{t=K+1}^T \mathbb{I}\left\{ \hat{\mu}_{1}(t-1)+\sqrt{\frac{6\log t+2\log(c\vee 1)}{N_{1}(t-1)}}  \le \mu_1 
\right\}\right] \notag \\
&\phantom{\ = \ }
+ \underbrace{\E\left[\sum_{t=K+1}^T \mathbb{I}\left\{ A_t=i \text { and } \left({\mu}_{1}-2\sqrt{\frac{6\log t+2\log(c\vee 1)}{N_{1}(t-1)}}\right) - \hat{\mu}_{i}(t-1) < c 
\right\}\right]}_{(\bigstar)_i}. \label{equation_fixedregret_asymptotic1}
\end{align}

Following a similar argument as in~\eqref{equation_fixedregret_result1}, we can derive 
\begin{equation}
 \E\left[\sum_{t=K+1}^T \mathbb{I}\left\{ \hat{\mu}_{1}(t-1)+\sqrt{\frac{6\log t+2\log(c\vee 1)}{N_{1}(t-1)}}  \le \mu_1 
\right\}\right] \le \frac 1 {Kc}. \label{equation_fixedregret_asymptotic2}
\end{equation}

Now we focus on the last term in~\eqref{equation_fixedregret_asymptotic1}, which is denoted by  $(\bigstar)_i$.

For any fixed $b\in(0,1)$, there must exist a constant $t_1(b, \mu, c)\ge K+1$ such that for all $t\ge t_1$,
$$
2\sqrt{\frac{6\log t+2\log(c\vee 1)}{(t-1)^b}} \le \frac {\Delta_i -c} 2  .
$$

Notice that for all $t\ge t_1$,
\begin{align*}
&\phantom{\ = \ } \left\{ A_t=i \text { and } \left({\mu}_{1}-2\sqrt{\frac{6\log t+2\log(c\vee 1)}{N_{1}(t-1)}}\right) - \hat{\mu}_{i}(t-1) < c 
\right\} \\ 
&\subseteq \left\{{N_{1}(t-1)} \le (t-1)^b\right\}  \\
&\phantom{\ = \ } \cup   \left\{ A_t=i \text { and } \left({\mu}_{1}-2\sqrt{\frac{6\log t+2\log(c\vee 1)}{N_{1}(t-1)}}\right) - \hat{\mu}_{i}(t-1) < c  \text { and } {N_{1}(t-1)} > (t-1)^b \right\} \\
&\subseteq \left\{{N_{1}(t-1)} \le (t-1)^b\right\} \cup   \left\{ A_t=i \text { and } \hat{\mu}_{i}(t-1) \ge \mu_1 + \frac {\Delta_i -c} 2  \right\}.
\end{align*}

From the above, we deduce that
\begin{align*}
(\bigstar)_i&\le t_1 + \sum_{t=t_1}^T \mathbb P \left({N_{1}(t-1)} \le (t-1)^b \right) + \E\left[\sum_{t=t_1}^T \mathbb{I}\left\{ A_t=i \text { and } \hat{\mu}_{i}(t-1) \ge \mu_1 + \frac {\Delta_i -c} 2 
\right\}\right].
\end{align*}

Using the approach similar to the one used to bound $(\blacksquare)_i$ in \eqref{equation_blacksquare2}, we have 
\begin{align*}
    \E\left[\sum_{t=t_1}^T \mathbb{I}\left\{ A_t=i \text { and } \hat{\mu}_{i}(t-1) \ge \mu_1 + \frac {\Delta_i -c} 2 
\right\}\right] \le \frac 8 {(\Delta_i - c)^2}.
\end{align*}

By applying Lemma~\ref{lemma_N1t}, we can get
$$
(\bigstar)_i \le t_1(b, \mu, c) + \beta(b, \mu, K) + \frac 8 {(\Delta_i - c)^2}
$$
where the term $\beta(b, \mu, K)$ is subsequently defined in Lemma~\ref{lemma_N1t}.

Substituting the above inequality and \eqref{equation_fixedregret_asymptotic2} into \eqref{equation_fixedregret_asymptotic1}, we arrive at
\begin{align*}
    \E[N_i^{(0)}(T)] &\le 1 + \frac 1 {Kc} + t_1(b, \mu, c) + \beta(b, \mu, K) + \frac 8 {(\Delta_i - c)^2} \\
    &= o(\log T).
\end{align*}

Due to the asymptotic optimality of Less-Exploring Thompson Sampling \citep{jin23b}, for any suboptimal arm $i$, we have
$$
\E[N_i(T)] \le \frac {2 \log T} {\Delta_i^2} + o(\log T).
$$

Finally, based on the regret decomposition in \eqref{equation_fixedregret_regretdecompose}, we can conclude 
\begin{align*}
R^{\FRG}_{\mu,c}(T) &= c\cdot \E[N_1^{(1)}(T)] +\sum_{i>1}  \left( \Delta_i \cdot \E[N_i^{(0)}(T)]  + c\cdot \E[N_i^{(1)}(T)]\right)\\
&= c\cdot \E[N_1^{(1)}(T)] + \sum_{i>1} ({\Delta_i\wedge c }) \cdot \E[N_i(T)] \\
&\phantom{\ = \ } + \sum_{i>1}  \left( (\Delta_i-{\Delta_i\wedge c }) \cdot \E[N_i^{(0)}(T)]  + (c-{\Delta_i\wedge c })\cdot \E[N_i^{(1)}(T)]\right) \\
&= \sum_{i>1} ({\Delta_i\wedge c }) \cdot \E[N_i(T)] + \sum_{i:\Delta_i < c} (c-{\Delta_i })\cdot \E[N_i^{(1)}(T)]  + \sum_{i:\Delta_i > c}(\Delta_i-{c })\cdot \E[N_i^{(0)}(T)]\\
&\le ( 2 \log T) \sum_{i>1} \frac{\Delta_i\wedge c }{\Delta_i^2}  + o(\log T)
\end{align*}
where the second equality is due to the fact that $\E[N_i^{(0)}(T)] + \E[N_i^{(1)}(T)] = \E[N_i(T)]$ for all arms $i\in[K]$.

Therefore, it holds that
$$
\limsup_{T\to \infty} \frac {R^{\FRG}_{\mu,c}(T)} {\log T} \le  2 \sum_{i>1} \frac{\Delta_i\wedge c }{\Delta_i^2} 
$$
as desired.
\end{proof}

\begin{lemma}
\label{lemma_N1t}
Consider Algorithm~\ref{algo1}. For any $b\in(0,1)$, there exists a constant $\beta(b, \mu, K) $ such that
$$
\sum_{t=1}^{\infty} \mathbb P \left({N_{1}(t)} \le t^b \right) \le \beta(b, \mu, K). 
$$
\end{lemma}

\begin{proof}
The proof of Lemma~\ref{lemma_N1t} closely follows that of Proposition~5 in \citet{korda2013thompson}, which was used to analyze the classical Thompson Sampling algorithm. In fact, the only difference between our arm sampling rule, which is built upon Less-Exploring Thompson Sampling \citep{jin23b}, and the classical Thompson Sampling is how the estimated reward $a_i(t)$ is constructed for each arm $i\in[K]$. Specifically, in our arm sampling rule, $a_i(t)$ is either drawn from the posterior distribution $\mathcal N (\hat{\mu}_i(t-1), 1/N_i(t-1))$ with probability $1/K$ or set to be the empirical mean $\hat{\mu}_i(t-1)$ otherwise. In classical Thompson Sampling, $a_i(t)$ is always drawn from the posterior distribution. Therefore, it suffices to verify the parts concerning the probability distributions of $a_i(t)$; these correspond to Lemmas~9 and 10 in the proof of Proposition~5 in \citet{korda2013thompson}. 

It is straightforward to see that Lemma~9 in \citet{korda2013thompson} is applicable to our algorithm. For Lemma~10 therein, its counterpart is demonstrated in Lemma~\ref{lemma_lemma_N1t} below.

After establishing the counterparts of Lemmas~9 and 10 in the proof of Proposition~5 in \citet{korda2013thompson}, we can extend the same analysis to our specific case. For the sake of completeness, we provide a proof sketch in the following. 

Let $\tau_j$ denote the time of the $j$-th pull of the optimal arm (i.e., arm $1$), with $\tau_0:=0$. Define $\xi_j:=\left(\tau_{j+1}-1\right)-\tau_j$ as the random variable measuring the number of time steps between the $j$-th and $(j+1)$-th pull of the optimal arm. With this setup, we can derive an upper bound for $\mathbb{P} \left({N_{1}(t)} \le t^b \right)$ as:
$$
\mathbb{P} \left({N_{1}(t)} \le t^b \right) \leq \mathbb{P} \left(\exists j \in\left\{0, . ., \left\lfloor t^b\right\rfloor \right\}: \xi_j \geq t^{1-b}-1 \right) \le \sum_{j=0}^{\left\lfloor t^b\right\rfloor} \mathbb{P}({\xi_j \geq t^{1-b}-1}).
$$

Consider the interval $\mathcal{I}_j:=\left\{\tau_j, \ldots, \tau_j+\left\lceil t^{1-b}-1\right\rceil\right\} $. If ${\xi_j \geq t^{1-b}-1}$, then no pull of the optimal arm occurs on  $\mathcal{I}_j$. 

The subsequent analysis aims to bound the probability that no pull of the optimal arm occurs within the interval $\mathcal{I}_j$. It relies on two key principles:
\begin{itemize}
    \item First, for a suboptimal arm, if it has been pulled a sufficient number of times, then, with high probability, its estimated reward (sample) cannot deviate significantly from its true mean.  This observation is quantitatively characterized in Lemma 10 of \citet{korda2013thompson}, corresponding to Lemma~\ref{lemma_lemma_N1t} in our paper.
    \item Second, for the optimal arm, the probability that its estimated reward (sample) deviates significantly below its true mean during a long subinterval of $\mathcal{I}_j$ is low. This observation is quantitatively characterized in Lemma 9 of \citet{korda2013thompson}, which directly applies to our case.
\end{itemize}
\end{proof}

\begin{lemma}
\label{lemma_lemma_N1t}
Consider Algorithm~\ref{algo1}. For all $t\in \mathbb N$, it holds that
$$
\mathbb P \left(\exists \, s \le t, \exists\,  i >1 : a_i(s) > \mu_i + \frac {\Delta_i} 2, N_i(s-1) > \frac {128\log t}{ \Delta_i^2} \right) \le \frac {K} {t^2}.
$$
\end{lemma}
\begin{proof}
For any fixed $s \le t$ and $ i >1$, we have 
\begin{equation}
\begin{aligned}
   &\phantom{\ = \ } \mathbb P \left( a_i(s) > \mu_i + \frac {\Delta_i} 2, N_i(s-1) > \frac {128\log t}{ \Delta_i^2} \right) \\
   &\le \mathbb P \left(\hat \mu_i(s-1) > \mu_i + \frac {\Delta_i} 4 , N_i(s-1) > \frac {128\log t}{ \Delta_i^2} \right)  \\
   &\phantom{\ = \ } + \mathbb P \left( a_i(s) > \mu_i + \frac {\Delta_i} 2, N_i(s-1) > \frac {128\log t}{ \Delta_i^2}, \hat \mu_i(s-1) \le \mu_i + \frac {\Delta_i} 4 \right).
   \label{equation_lemma_lemma_N1t}
\end{aligned}
\end{equation}

For the first term in~\eqref{equation_lemma_lemma_N1t}, by Lemma~\ref{lemma_hoeffding1}, we can bound it as 
\begin{align*}
       &\phantom{\ = \ } \mathbb P \left(\hat \mu_i(s-1) > \mu_i + \frac {\Delta_i} 4 , N_i(s-1) > \frac {128\log t}{ \Delta_i^2} \right)  \\
       &\le  \sum_{x= \left \lceil \frac {128\log t}{ \Delta_i^2} \right\rceil} ^ t \mathbb P \left(\hat \mu_{ix} > \mu_i + \frac {\Delta_i} 4 , N_i(s-1) = x\right) \\
       &\le  \sum_{x= \left \lceil \frac {128\log t}{ \Delta_i^2} \right\rceil} ^ t \mathbb P \left(\hat \mu_{ix} > \mu_i + \frac {\Delta_i} 4 \right) \\
       &\le \sum_{x= \left \lceil \frac {128\log t}{ \Delta_i^2} \right\rceil} ^ t \frac 1 {t^4} \\
       &\le \frac 1 {t^3}.
\end{align*}

For the second term in~\eqref{equation_lemma_lemma_N1t}, according to the construction of $a_i(s)$ in Algorithm~\ref{algo1} and Lemma~\ref{lemma_hoeffding1}, we can get
\begin{align*}
    &\phantom{\ = \ }  \mathbb P \left( a_i(s) > \mu_i + \frac {\Delta_i} 2, N_i(s-1) > \frac {128\log t}{ \Delta_i^2}, \hat \mu_i(s-1) \le \mu_i + \frac {\Delta_i} 4 \right)\\
    &\le \mathbb P \left( a_i(s) > \hat \mu_i(s-1) + \frac {\Delta_i} 4, N_i(s-1) > \frac {128\log t}{ \Delta_i^2}\right) \\
    &\le  \sum_{x= \left \lceil \frac {128\log t}{ \Delta_i^2} \right\rceil} ^ t \mathbb P \left( a_i(s) > \hat \mu_i(s-1) + \frac {\Delta_i} 4, N_i(s-1) = x \right) \\
    &\le  \sum_{x= \left \lceil \frac {128\log t}{ \Delta_i^2} \right\rceil} ^ t  \frac 1 K  \cdot   \frac 1 {t^4}  \\
    &\le \frac 1 {K t^3}.
\end{align*}

Thus, for any $s \le t$ and $ i >1$, it holds that 
\begin{equation*}
   \mathbb P \left( a_i(s) > \mu_i + \frac {\Delta_i} 2, N_i(s-1) > \frac {128\log t}{ \Delta_i^2} \right) \le \frac 1 {t^3} + \frac 1 {K t^3} = \frac {K+1} {K t^3}.
\end{equation*}

Finally, by a union bound, we can conclude 
$$
\mathbb P \left(\exists \, s \le t, \exists \, i >1 : a_i(s) > \mu_i + \frac {\Delta_i} 2, N_i(s-1) > \frac {128\log t}{ \Delta_i^2} \right) \le \frac {(K+1)(K-1)} {K t^2}\le  \frac {K} {t^2}.
$$
\end{proof}


\subsection{Lower Bounds}
\label{subsection_fixedregret_lowerboundproof}
\begin{proof}[Proof of Theorem~\ref{theorem_fixedregret_lowerbound}]
In the following, we will establish the asymptotic and minimax lower bounds, respectively.

\paragraph{Asymptotic lower bound. }  Consider any algorithm $\pi$ that is $R^{\FRG}$-consistent. 
Sine $\E[N_i^{(0)}(T)] + \E[N_i^{(1)}(T)] = \E[N_i(T)]$ for all arms $i\in[K]$, we can utilize the regret decomposition in \eqref{equation_fixedregret_regretdecompose} to derive
\begin{align*}
    R^{\FRG}_{\mu,c}(T, \pi) &\ge \sum_{i>1}  \left( \Delta_i \cdot \E[N_i^{(0)}(T)]  + c\cdot \E[N_i^{(1)}(T)]\right) \\
    &\ge \sum_{i>1}   \Big( (\Delta_i \wedge c) \cdot \E[N_i(T)]\Big).
\end{align*}

Fix any abstention regret $c>0$. Then for all bandit instances $\mu\in \mathcal U$ and $a>0$, it holds that 
\begin{align*}
R^{\CA}_{\mu}(T, \pi) & = T\mu_1 - \mathbb E\left[\sum_{t=1}^T X_t \right] \\
&= \sum_{i>1}  \Big( \Delta_i \cdot \E[N_i(T)] \Big)\\
&\le \max_{i>1} \frac {\Delta_i} {\Delta_i \wedge c} \cdot \sum_{i>1}   \Big( (\Delta_i \wedge c) \cdot \E[N_i(T)]\Big) \\
&\le \max_{i>1} \frac {\Delta_i} {\Delta_i \wedge c} \cdot R^{\FRG}_{\mu,c}(T, \pi)  \\
&= o(T^a) .
\end{align*}

Therefore, in accordance with Definition~\ref{definition_standard_consistent}, the algorithm $\pi$ is also $R^{\CA}$-consistent for arbitrary abstention regret $c$. 

Subsequently, for any abstention regret $c$ and bandit instance $\mu$, we have 
\begin{align*}
    \liminf_{T\to \infty} \frac {R^{\FRG}_{\mu,c}(T,\pi)} {\log T} 
    &\ge  \liminf_{T\to \infty}  \frac{\sum_{i>1}\left( (\Delta_i \wedge c) \cdot \E[N_i(T)]\right)} {\log T}  \\
    &= \sum_{i>1} (\Delta_i \wedge c) \cdot \liminf_{T\to \infty} \frac{  \E[N_i(T)]} {\log T}  \\
    &\ge 2 \sum_{i>1} \frac{\Delta_i\wedge c }{\Delta_i^2} ,
\end{align*}
where the last inequality follows from the property of $R^{\CA}$-consistent policies as detailed in Remark~\ref{remark_optimal_property}.

This concludes the proof of the instance-dependent asymptotic lower bound.

\paragraph{Minimax lower bound. } We extend the proof of the minimax lower bound from the canonical multi-armed bandit model to the model incorporating fixed-regret abstention.

Consider any fixed abstention regret $c > 0$, time horizon $T\ge K$ and algorithm $\pi \in \Pi^{\FRG}$. We construct a bandit instance $\mu\in \mathcal U$, where $\mu_1 = \Delta$ and $\mu_i=0$ for all $i\in[K]\setminus \{1\}$. Here,  $\Delta >0$ is some parameter whose exact value will be determined later. We use $\mathbb P_{\mu,c}$ to represent the probability distribution of the sequence $(A_1,B_1,X_1,\ldots,A_T,B_T, X_T)$ induced by the algorithm $\pi$ for the abstention regret $c$ and bandit instance $\mu$. Since $\sum_{i=1} \E_{\mu,c}[N_i(T)] =T$, according to  the pigeonhole principle, there must exist an index $j\in[K]\setminus \{1\}$ such that 
$$
\E_{\mu,c}[N_j(T)] \le \frac {T} {K-1}.
$$
Now we construct another bandit instance  $\mu' \in \mathcal U$, where $\mu'_1 = \Delta$, $\mu'_j= 2\Delta$ and $\mu'_i=0$ for all $i\in[K]\setminus \{1, j\}$.  Let $\mathbb P_{\mu',c}$ denote the probability distribution of the sequence $(A_1,B_1,X_1,\ldots,A_T,B_T, X_T)$ induced by the algorithm $\pi$ for the abstention regret $c$ and bandit instance $\mu'$.

For the first bandit instance $\mu$, regardless of the abstention option, if $N_1(T)\le T/2$, then the cumulative regret must be at least $(\Delta\wedge c)T/2$. Therefore, we have 
$$
{R^{\FRG}_{\mu,c}(T,\pi)} \ge  \frac {(\Delta\wedge c)T} 2 \mathbb P_{\mu,c} (N_1(T)\le T/2).
$$
Similarly, for the second bandit instance $\mu'$, we can obtain
$$
{R^{\FRG}_{\mu',c}(T,\pi)} \ge  \frac {(\Delta\wedge c)T} 2 \mathbb P_{\mu',c} (N_1(T)> T/2).
$$

By combining the aforementioned two inequalities and applying Lemma~\ref{lemma_BH}, we obtain the following:
\begin{align*}
    {R^{\FRG}_{\mu,c}(T,\pi)} + {R^{\FRG}_{\mu',c}(T,\pi)} &\ge \frac {(\Delta\wedge c)T} 2 \left(\mathbb P_{\mu,c} (N_1(T)\le T/2)+ \mathbb P_{\mu',c} (N_1(T)> T/2)\right) \\
    &\ge \frac {(\Delta\wedge c)T} 4 \exp \left(-\mathrm{KL}\left(\mathbb P_{\mu,c}, \mathbb P_{\mu',c}\right)\right).
\end{align*}

Leveraging Lemma~\ref{lemma_divergence_decomposition} and  the KL divergence between  Gaussian distributions, we can derive
$$
\mathrm{KL}\left(\mathbb P_{\mu,c}, \mathbb P_{\mu',c}\right) = \E_{\mu,c}[N_j(T)] \frac{(2\Delta)^2} 2 \le \frac {2T\Delta^2} {K-1}.
$$

Altogether, we can arrive at 
$$
{R^{\FRG}_{\mu,c}(T,\pi)} + {R^{\FRG}_{\mu',c}(T,\pi)} \ge  \frac {(\Delta\wedge c)T} 4 \exp \left(-\frac {2T\Delta^2} {K-1} \right).
$$

Now, we set $\Delta = \sqrt{\frac K T} \wedge c$, which leads to
\begin{align*}
    {R^{\FRG}_{\mu,c}(T,\pi)} + {R^{\FRG}_{\mu',c}(T,\pi)} &\ge \frac{\sqrt{KT}\wedge cT} 4 \exp \left(-\frac {2T(\sqrt{\frac K T} \wedge c)^2} {K-1} \right) \\
    &\ge \frac{\sqrt{KT}\wedge cT} 4 \exp \left(-\frac {2K} {K-1} \right) \\
    &\ge \frac{\exp(-4)} 4  (\sqrt{KT}\wedge cT).
\end{align*}

Consequently, either ${R^{\FRG}_{\mu,c}(T,\pi)}$ or ${R^{\FRG}_{\mu',c}(T,\pi)} $ is at least $\frac{\exp(-4)} 8  (\sqrt{KT}\wedge cT)$, which completes the proof of the instance-independent minimax lower bound.
\end{proof}

\subsection{Heterogeneous Abstention Regret}
\label{appendix_hetorogenous_abstention}
In the scenario of heterogeneous abstention regret, for each arm $i\in[K]$, the agent incurs a regret of $c_i>0$ when pulling arm $i$ and choosing abstention.  By the law of total expectation,  the expected cumulative regret up to the time horizon $T$ can be expressed as:
\begin{align*}
    R^{\FRG}_{\mu, \mathbf{c}}(T, \pi) &=\mathbb E\left[\sum_{t=1}^T\big ((\mu_1-X_t)\cdot\mathbf 1\{B_t=0\}+c_{A_t}\cdot\mathbf 1\{B_t=1\}\big )\right]  \\
    &= c_1 \cdot \E[N_1^{(1)}(T)] +\sum_{i>1}  \left( \Delta_i \cdot \E[N_i^{(0)}(T)]  + c_i \cdot \E[N_i^{(1)}(T)]\right). 
\end{align*}

For ease of presentation, we define $c = \max_{i\in[K]} c_i$ as the maximum abstention regret level. The lower bound part is straightforward. Specifically, we can obtain the following lower bounds.

\begin{theorem} 
\label{theorem_fixedregret_lowerbound_extension}
For any abstention regret $\mathbf c = (c_i)_{i\in[K]}$, bandit instance $\mu\in\mathcal U$ and $R^{\FRG}$-consistent algorithm~$\pi$, it holds that
$$
\liminf_{T\to \infty} \frac {R^{\FRG}_{\mu,\mathbf{c}}(T,\pi)} {\log T} \ge  2 \sum_{i>1} \frac{\Delta_i\wedge c_i }{\Delta_i^2} . 
$$
For any maximum abstention regret level $c > 0$ and time horizon $T\ge K$, there exists a universal constant $\alpha > 0$ such that
$$
\inf_{\pi\in \Pi^{\FRG}} \sup _{\mu\in \mathcal U, \mathbf{c} } {R^{\FRG}_{\mu,\mathbf{c}}(T,\pi)}  \ge \alpha (\sqrt{KT}\wedge cT).
$$
\end{theorem}

With respect to the algorithm design, for each arm $i\in [K]$, we redefine its {lower confidence bound} as
\begin{equation*}
\mathrm{LCB}_i(t) = \hat{\mu}_{i}(t-1)-\sqrt{\frac{6\log t+2\log(c_i\vee 1)}{N_{i}(t-1)}}.
\end{equation*}
Then in the abstention decision rule, we choose to abstain ($B_t = 1$) if and only if 
\begin{equation}
\label{equation_new_rule}
\max_{i\in [K]\setminus \{A_t\}} \mathrm{LCB}_i(t) - \hat{\mu}_{A_t}(t-1) \ge c_i \ \text{ or }\ \sqrt{\frac K t} \ge c.
\end{equation}
Note that in \eqref{equation_new_rule}, we compare $\sqrt{\frac K t}$ with $c=\max_{i\in[K]} c_i$.

\begin{theorem}
\label{theorem_fixedregret_upperbound_extension}
For any abstention regret $\mathbf c = (c_i)_{i\in[K]}$ and bandit instance $\mu\in\mathcal U$, Algorithm~\ref{algo1} with the new abstention criteria in  \eqref{equation_new_rule} guarantees that 
$$
\limsup_{T\to \infty} \frac {R^{\FRG}_{\mu,\mathbf{c}}(T)} {\log T} \le  2 \sum_{i>1} \frac{\Delta_i\wedge c_i }{\Delta_i^2} .
$$
Furthermore, there exists a universal constant $\alpha > 0$ such that
$$
{R^{\FRG}_{\mu,\mathbf{c}}(T)}  \le  \begin{cases}
     cT  & \text{ if }  c\le \sqrt{ K/ T} \\
   \alpha  (\sqrt{KT}+\sum_{i>1}{\Delta_i}) & \text{ if }  c> \sqrt{ K /T}.
\end{cases}
$$
\end{theorem}

The proof of Theorem~\ref{theorem_fixedregret_upperbound_extension} closely follows that of Theorem~\ref{theorem_fixedregret_upperbound} in Appendix~\ref{subsection_fixedregret_upperboundproof}. While many steps are analogous, it is noteworthy that in this context, for any arm $i$ with $\Delta_i <c_i$, we can upper bound $\E[N_i^{(1)}(T)]$ as 
\begin{align*}
\E[N_i^{(1)}(T)] &\le \E\left[\sum_{t=1}^T \mathbb{I}\left\{ A_t=i \text { and } \sqrt{\frac K t} \ge c  \right\}\right]    \notag \\
&\phantom{\ = \ } +\E\left[\sum_{t=K+1}^T  \mathbb{I}\left\{  A_t=i \text { and } \max_{j\in [K]} \left(\hat{\mu}_{j}(t-1)-\sqrt{\frac{6\log t+2\log(c_j\vee 1)}{N_{j}(t-1)}}\right) \ge \mu_1 \right\}\right]  \notag \\
&\phantom{\ = \ } +\E\left[\sum_{t=K+1}^T \mathbb{I}\left\{ A_t=i \text { and }  \mu_1-\hat{\mu}_i(t-1) \ge c_i \right\}\right].
\end{align*}

\section{Analysis of the Fixed-Reward Setting}
\subsection{Upper Bounds}
\label{subsection_fixedreward_upperboundproof}
\begin{proof}[Proof of Theorem~\ref{theorem_fixedreward_upperbound}]
Utilizing the law of total expectation, we can decompose the regret $R^{\FRW}_{\mu,c}(T, \pi)$ in the following:
\begin{align}
R^{\FRW}_{\mu,c}(T, \pi) &=T \cdot \left( \mu_1\vee c \right)- \mathbb E\left[\sum_{t=1}^T\Big(X_t\cdot\mathbf 1\{B_t=0\}+c\cdot\mathbf 1\{B_t=1\}\Big)\right] \notag \\
&=T \cdot \left( \mu_1\vee c \right)- \mathbb E\left[\sum_{t=1}^T\Big(\mu_{A_t}\cdot\mathbf 1\{B_t=0\}+c\cdot\mathbf 1\{B_t=1\}\Big)\right] \notag \\
&= \sum_{i\in[K]} \left((\mu_1\vee c - \mu_i)\cdot \E[N_i^{(0)}(T)] + (\mu_1\vee c - c)\cdot \E[N_i^{(1)}(T)]\right).
\label{equation_fixedreward_regretdecompose}
\end{align}

Recall that we define $\hat{\mu}_i(t) = +\infty $ if $N_i(t)= 0$ for all arms $i\in[K]$. Thus, for any arm $i$ with $\mu_i > c$, we can obtain 
$$
\begin{aligned}
\E[N_i^{(1)}(T)] &=\E\left[\sum_{t=1}^T \mathbb{I}\left\{ A_t=i \text { and } B_t=1 \right\}\right] \\
&= \E\left[\sum_{t=1}^T \mathbb{I}\left\{A_t=i \text { and }  \hat{\mu}_i(t-1) \le c \right\}\right] \\
&\le \E\left[\sum_{t=1}^T \sum_{s=0}^{T-1} \mathbb{I}\left\{ 
  A_t=i \text { and } \hat{\mu}_{is}\le c \text { and }  N_i(t-1)=s\right\}\right]\\
  &=\E\left[\sum_{t=1}^T \sum_{s=1}^{T-1} \mathbb{I}\left\{ 
  A_t=i \text { and } \hat{\mu}_{is}\le c \text { and }  N_i(t-1)=s\right\}\right] \\
  &\le  \E\left[\sum_{s=1}^{T-1} \mathbb{I}\left\{ 
  \hat{\mu}_{is}\le c \right\}\right] \\
  &\le \frac 2 {(\mu_i-c)^2}
\end{aligned}
$$
where the penultimate inequality is derived from an argument analogous to that in \eqref{equation_blacksquare1}, and the last inequality is a consequence of Lemma~\ref{lemma_hoeffding2}.

Similarly, for any arm $i$ with $\mu_i < c$, we have 
$$
\begin{aligned}
\E[N_i^{(0)}(T)] &=\E\left[\sum_{t=1}^T \mathbb{I}\left\{ A_t=i \text { and } B_t=0 \right\}\right] \\
&= \E\left[\sum_{t=1}^T \mathbb{I}\left\{A_t=i \text { and }  \hat{\mu}_i(t-1) > c \right\}\right] \\
&\le \E\left[\sum_{t=1}^T \sum_{s=0}^{T-1} \mathbb{I}\left\{ 
  A_t=i \text { and } \hat{\mu}_{is}> c \text { and }  N_i(t-1)=s\right\}\right]\\
&\le 1 + \E\left[\sum_{t=1}^T \sum_{s=1}^{T-1} \mathbb{I}\left\{ 
  A_t=i \text { and } \hat{\mu}_{is}> c \text { and }  N_i(t-1)=s\right\}\right]\\
  &\le 1+ \frac 2 {(c-\mu_i)^2}.
\end{aligned}
$$

\paragraph{Asymptotic upper bound. } First, we consider the scenario where $\mu_1 \le c$. In this case, based on the regret decomposition in Equation~\eqref{equation_fixedreward_regretdecompose}, we can bound the regret as follows:
\begin{align}
R^{\FRW}_{\mu,c}(T)&=\sum_{i\in[K]} (c-\mu_i)\cdot \E[N_i^{(0)}(T)] \notag \\
    &= \sum_{i:\mu_i<c} (c-\mu_i)\cdot \E[N_i^{(0)}(T)] \label{equation_fixedreward_decomposition1}\\
    &\le \sum_{i:\mu_i<c} \left(c-\mu_i + \frac 2 {c-\mu_i} \right)  \notag \\
    &=o(\log T). \notag
\end{align}

Next, we consider the scenario where $\mu_1 > c$. Due to the asymptotic optimality of the base algorithm, for any suboptimal arm $i$, we have
$$
\E[N_i(T)] \le \frac {2 \log T} {\Delta_i^2} + o(\log T).
$$
Thus, we can bound the regret as: 
\begin{align*}
&\phantom{\ =\ \ \ }R^{\FRW}_{\mu,c}(T)\\ 
&= \sum_{i\in[K]} \left((\mu_1 - \mu_i)\cdot \E[N_i^{(0)}(T)] + (\mu_1 - c)\cdot \E[N_i^{(1)}(T)]\right) \\
&=  \sum_{i\in[K]} \left((\mu_1 - \mu_i\vee c + \mu_i\vee c - \mu_i)\cdot \E[N_i^{(0)}(T)] + (\mu_1 - \mu_i\vee c + \mu_i\vee c-c)\cdot \E[N_i^{(1)}(T)]\right) \\
&= \sum_{i\in[K]} \left((\mu_1 - \mu_i\vee c) \cdot \E[N_i(T)] +  ( \mu_i\vee c - \mu_i)\cdot \E[N_i^{(0)}(T)] + (\mu_i\vee c-c)\cdot \E[N_i^{(1)}(T)]\right) \\
&= \sum_{i>1} (\mu_1 - \mu_i\vee c) \cdot \E[N_i(T)] + \sum_{i:\mu_i<c} (c-\mu_i)  \cdot \E[N_i^{(0}(T)] + \sum_{i:\mu_i>c} (\mu_i-c)  \cdot \E[N_i^{(1)}(T)] \\
&\le  (2\log T) \sum_{i>1} \frac{ \mu_1 - \mu_i\vee c }{\Delta_i^2} +  \sum_{i:\mu_i<c} \left(c-\mu_i + \frac 2 {c-\mu_i} \right) +\sum_{i:\mu_i>c} \frac 2 {\mu_i-c}+ o(\log T) \\
&= (2\log T) \sum_{i>1} \frac{ \mu_1 - \mu_i\vee c }{\Delta_i^2} + o(\log T)
\end{align*}
where the third equality is due to the fact that $\E[N_i^{(0)}(T)] + \E[N_i^{(1)}(T)] = \E[N_i(T)]$ for all arms $i\in[K]$. 

Altogether, in both scenarios, it holds that
$$
\limsup_{T\to \infty} \frac {R^{\FRW}_{\mu,c}(T)} {\log T} \le  2 \sum_{i>1} \frac{ \mu_1\vee c - \mu_i\vee c }{\Delta_i^2}.
$$

\paragraph{Minimax upper bound. } First, if $\mu_1\le c$, by utilizing the regret decomposition in Equation~\eqref{equation_fixedreward_decomposition1}, we have 
\begin{equation}
\label{equation_fixedreward_decomposition2}
\begin{aligned}
R^{\FRW}_{\mu,c}(T) &= \sum_{i:\mu_i<c} (c-\mu_i)\cdot \E[N_i^{(0)}(T)]  \\
&= \sum_{i:0<c-\mu_i<\sqrt{\frac K T}}  (c-\mu_i)\cdot \E[N_i^{(0)}(T)] + \sum_{i:c-\mu_i\ge \sqrt{\frac K T}}  (c-\mu_i)\cdot \E[N_i^{(0)}(T)] \\
&\le \sqrt{\frac K T} \sum_{i:0<c-\mu_i<\sqrt{\frac K T}} \E[N_i^{(0)}(T)] + \sum_{i:c-\mu_i\ge \sqrt{\frac K T}} \left(c-\mu_i + \frac 2 {c-\mu_i} \right)\\
&\le  \sqrt{\frac K T} \cdot T  + \sum_{i:c-\mu_i\ge \sqrt{\frac K T}} \left(c-\mu_i + 2\sqrt{\frac T K}\right)\\
&\le  \sqrt{ K T}   + \sum_{i\in[K]} \left(c-\mu_i\right)+ K\cdot 2\sqrt{\frac T K} \\
&=3\sqrt{KT} +  \sum_{i\in[K]} \left( \mu_1\vee c -\mu_i\right).
\end{aligned}
\end{equation}

Next, if $\mu_1> c$, then the best possible (expected) reward at a single time step is $\mu_1$, which coincides with the canonical multi-armed bandit problem. Consequently, compared with canonical multi-armed bandits, at a single time step, the agent in our problem incurs a greater (expected) regret only if an arm $i$ with $\mu_i>c$ is pulled and the abstention option is chosen.  Thus, we have 
$$
R^{\FRW}_{\mu,c}(T) \le R^{\CA}_{\mu}(T) + \sum_{i:\mu_i>c} (\mu_i-c)  \cdot \E[N_i^{(1)}(T)].
$$

Due to the minimax optimality of the base algorithm, 
there exists a universal constant $\alpha_1 > 0$ such that
$$
R^{\CA}_{\mu}(T) \le \alpha_1 \left(\sqrt{KT}+\sum_{i>1}{\Delta_i}\right).
$$

Furthermore, using a similar argument as in \eqref{equation_fixedreward_decomposition2}, we can derive
\begin{align*}
    &\phantom{\ = \ } \sum_{i:\mu_i>c} (\mu_i-c)  \cdot \E[N_i^{(1)}(T)] \\
    &= \sum_{i:0<\mu_i-c<\sqrt{\frac K T}}  (\mu_i-c)\cdot \E[N_i^{(1)}(T)] + \sum_{i:\mu_i-c \ge \sqrt{\frac K T}}  (\mu_i-c)\cdot \E[N_i^{(1)}(T)] \\ 
    &\le \sqrt{KT} + \sum_{i:\mu_i-c \ge \sqrt{\frac K T}}   \frac 2 {\mu_i-c}\\
    &\le 3\sqrt{KT}.
\end{align*}

Therefore, we can bound the regret as
$$
\begin{aligned}
R^{\FRW}_{\mu,c}(T) &\le (\alpha_1+3) \left(\sqrt{KT}+\sum_{i>1}{\Delta_i}\right)\\
&= (\alpha_1+3) \left(\sqrt{KT}+\sum_{i\in[K]}\left( \mu_1\vee c -\mu_i\right)\right).
\end{aligned}
$$

As a result, the desired minimax upper bound holds in both scenarios.
\end{proof}

\subsection{Lower Bounds}
\label{subsection_fixedreward_lowerboundproof}
\begin{proof}[Proof of Theorem~\ref{theorem_fixedreward_lowerbound}]
The proof structure for Theorem~\ref{theorem_fixedreward_lowerbound} closely parallels that of Theorem~\ref{theorem_fixedregret_lowerbound} in Appendix~\ref{subsection_fixedregret_lowerboundproof}, although certain specific details contain significant variations. Therefore, we will streamline the shared components and elaborate on the distinctions.

\paragraph{Asymptotic lower bound. }  Consider any $R^{\FRW}$-consistent algorithm $\pi$ and bandit instance $\mu \in \mathcal U$. The case that $c\ge \mu_1 $ is trivial, as $R^{\FRW}_{\mu,c}(T, \pi)$ is non-negative, and $\mu_1\vee c - \mu_i\vee c=0$ for all $i > 1$. Thus, it suffices to demonstrate that for any abstention reward $c <\mu_1$, 
$$
\liminf_{T\to \infty} \frac {R^{\FRW}_{\mu,c}(T, \pi)} {\log T} \ge  2 \sum_{i>1} \frac{ \mu_1 - \mu_i\vee c }{\Delta_i^2}.
$$

When $c <\mu_1$, we can  establish a lower bound on $R^{\FRW}_{\mu,c}(T, \pi)$ as follows:
\begin{align*}
    R^{\FRW}_{\mu,c}(T, \pi) &= \sum_{i\in[K]} \left((\mu_1- \mu_i)\cdot \E[N_i^{(0)}(T)] + (\mu_1 - c)\cdot \E[N_i^{(1)}(T)]\right) \\
    &\ge \sum_{i>1} \left((\mu_1- \mu_i)\cdot \E[N_i^{(0)}(T)] + (\mu_1 - c)\cdot \E[N_i^{(1)}(T)]\right) \\
    &\ge \sum_{i>1} \Big((\mu_1 - \mu_i\vee c )\cdot \E[N_i(T)] \Big).
\end{align*}

Therefore, we only need to show for all suboptimal arms $i>1$, 
\begin{equation}
\label{equation_fixedreward_lowerbound_needed}
    \liminf _{T \rightarrow \infty} \frac{\E[N_i(T)]}{\log T } \geq \frac{2}{\Delta_i^2}.
\end{equation}

{However, unlike the asymptotic lower bound part of the proof of Theorem~\ref{theorem_fixedregret_lowerbound}, we cannot apply the properties of $R^{\CA}$-consistency here, as $R^{\FRW}$-consistency does not imply $R^{\CA}$-consistency in general. Instead, we will demonstrate the desired result \eqref{equation_fixedreward_lowerbound_needed} directly.}

Fix an index $j>1$ and take $\varepsilon >0$.  We now proceed to create an alternative bandit instance  $\mu' \in \mathcal U$, where $\mu'_j = \mu_1+\varepsilon$ and $\mu'_i=\mu_i$ for all $i\in[K]\setminus \{j\}$.  Note that for the new bandit instance, it holds that $\max_{i\in[K]}\mu'_i = \mu'_j > \mu_1 >c$.
To distinguish between the two scenarios, we will refer to the probability distribution associated with the sequence $(A_1, B_1, X_1, \ldots, A_T, B_T, X_T)$, generated by the algorithm $\pi$ for the abstention reward $c$ and the original bandit scenario $\mu$, as $\mathbb{P}_{\mu,c}$, and for the new bandit instance $\mu'$, we denote the corresponding distribution as $\mathbb{P}_{\mu',c}$.

A straightforward computation yields
$$
{R^{\FRW}_{\mu,c}(T,\pi)} \ge  \frac {(\mu_1 - \mu_i\vee c )T} 2 \mathbb P_{\mu,c} (N_j(T)> T/2)
$$
and
$$
{R^{\FRW}_{\mu',c}(T,\pi)} \ge  \frac {\varepsilon T} 2 \mathbb P_{\mu',c} (N_j(T) \le T/2).
$$

Employing a similar approach to the one used in the minimax lower bound part of the proof of Theorem~\ref{theorem_fixedregret_lowerbound}, utilizing Lemmas~\ref{lemma_BH} and \ref{lemma_divergence_decomposition}, we can derive:
\begin{align*}
    {R^{\FRW}_{\mu,c}(T,\pi)} + R^{\FRW}_{\mu',c}(T,\pi)&\ge  \frac {((\mu_1 - \mu_i\vee c )\wedge \varepsilon)T} 2 \left(\mathbb P_{\mu,c} (N_j(T)> T/2)+  \mathbb P_{\mu',c} (N_j(T) \le T/2) \right) \\
    &\ge  \frac {((\mu_1 - \mu_i\vee c )\wedge \varepsilon)T} 4 \exp \left(-\mathrm{KL}\left(\mathbb P_{\mu,c}, \mathbb P_{\mu',c}\right)\right) \\
    &=  \frac {((\mu_1 - \mu_i\vee c )\wedge \varepsilon)T} 4 \exp \left(-\E_{\mu,c}[N_j(T)] \frac{(\Delta_j+\varepsilon)^2} 2 \right).
\end{align*}

By rearranging the above inequality and taking the limit inferior, we have
\begin{align*}
    \liminf _{T \rightarrow \infty} \frac{\E[N_j(T)]}{\log T } &\ge \frac 2 {(\Delta_j+\varepsilon)^2} \left( 1+  \limsup _{T \rightarrow \infty} \frac {\log \left( \frac {(\mu_1 - \mu_i\vee c )\wedge \varepsilon} {4\left( {R^{\FRW}_{\mu,c}(T,\pi)} + R^{\FRW}_{\mu',c}(T,\pi)\right)}\right)} {\log T }\right) \\
    &= \frac 2 {(\Delta_j+\varepsilon)^2} \left( 1-  \limsup _{T \rightarrow \infty} \frac {\log \left( {R^{\FRW}_{\mu,c}(T,\pi)} + R^{\FRW}_{\mu',c}(T,\pi)\right)} {\log T }\right) .
\end{align*}

Recall the definition of $R^{\FRW}$-consistency. For all $a>0$, both $R^{\FRW}_{\mu,c}(T, \pi)$ and $R^{\FRW}_{\mu',c}(T,\pi)$ are on the order of $o(T^a)$, and hence,
$$
\limsup _{T \rightarrow \infty} \frac {\log \left( {R^{\FRW}_{\mu,c}(T,\pi)} + R^{\FRW}_{\mu',c}(T,\pi)\right)} {\log T } \le a.
$$

By letting both $a$ and $\varepsilon$ approach zero, we can establish the desired result \eqref{equation_fixedreward_lowerbound_needed}, thereby concluding the proof of the asymptotic lower bound.

\paragraph{Minimax lower bound. } The construction employed in the fixed-reward setting here is analogous to the one utilized in the proof of Theorem~\ref{theorem_fixedregret_lowerbound}. 

Consider any fixed abstention reward $c \in \mathbb R$, time horizon $T\ge K$ and algorithm $\pi \in \Pi^{\FRW}$.  Let $\Delta >0$ be a parameter to be determined later. We construct a bandit instance $\mu\in \mathcal U$, where $\mu_1 = \Delta + c$ and $\mu_i= c$ for all $i\in[K]\setminus \{1\}$.
Note that there must exist an index $j\in[K]\setminus \{1\}$ such that $\E_{\mu,c}[N_j(T)] \le \frac {T} {K-1}$.
We then construct another bandit instance  $\mu' \in \mathcal U$, where $\mu'_1 = \Delta +c$, $\mu'_j= 2\Delta+c$ and $\mu'_i=c$ for all $i\in[K]\setminus \{1, j\}$. 

Similarly, by applying Lemmas~\ref{lemma_BH} and \ref{lemma_divergence_decomposition}, we can derive that
\begin{align*}
    {R^{\FRW}_{\mu,c}(T,\pi)} + {R^{\FRW}_{\mu',c}(T,\pi)} &\ge \frac {\Delta T} 2 \left(\mathbb P_{\mu,c} (N_1(T)\le T/2)+ \mathbb P_{\mu',c} (N_1(T)> T/2)\right) \\
    &\ge \frac {\Delta T} 4 \exp \left(-\mathrm{KL}\left(\mathbb P_{\mu,c}, \mathbb P_{\mu',c}\right)\right) \\
    &= \frac {\Delta T} 4 \exp \left(- \E_{\mu,c}[N_j(T)] \frac{(2\Delta)^2} 2 \right) \\
    &\ge  \frac {\Delta T} 4 \exp \left(-\frac {2T\Delta^2} {K-1} \right).
\end{align*}

By choosing $\Delta = \sqrt{\frac K T} $, we have
\begin{align*}
    {R^{\FRW}_{\mu,c}(T,\pi)} + {R^{\FRW}_{\mu',c}(T,\pi)}&\ge \frac{\exp(-4)} 8  \sqrt{KT}.
\end{align*}

Consequently, either ${R^{\FRW}_{\mu,c}(T,\pi)}$ or ${R^{\FRW}_{\mu',c}(T,\pi)} $ is at least $\frac{\exp(-4)} 8  \sqrt{KT}$. 

Therefore, we have established the instance-independent minimax lower bound.
\end{proof}

\section{Numerical Experiments}
\label{appendix_experiment}

\subsection{Results for the Fixed-Reward Setting}
\label{subappendix_fixedreward_exp}

\begin{figure}[t]
\centering
	\begin{minipage}{0.79\linewidth}
        \hspace{0.00pt}
		\subfigure[Instance $\mu^{\dagger}$]{
			\begin{minipage}[b]{0.45\textwidth}
				\includegraphics[width=1.000\textwidth]{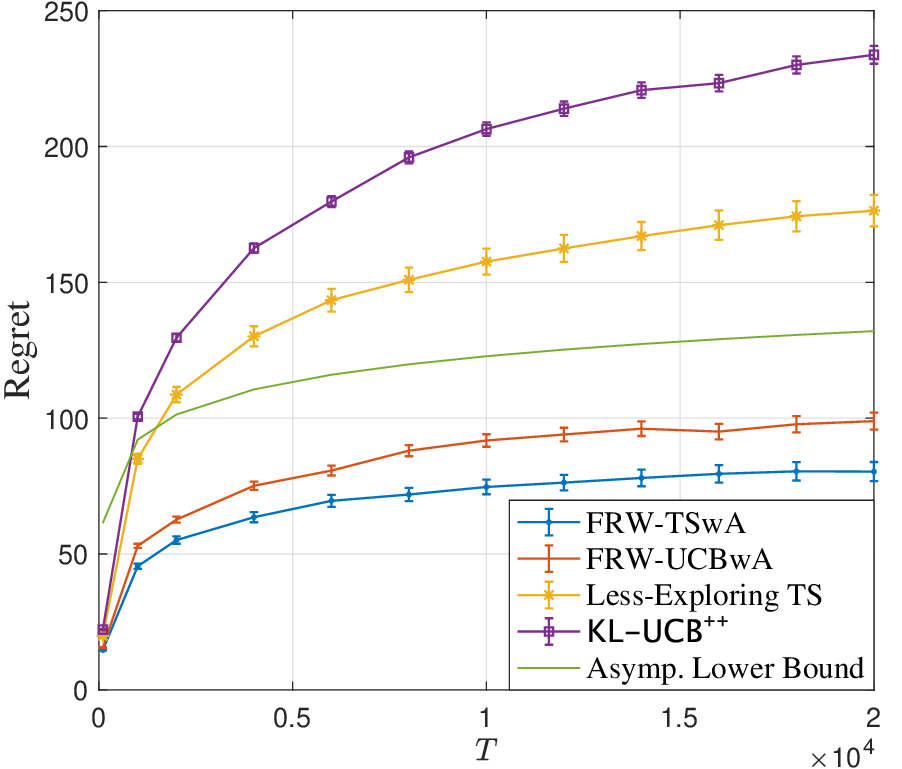}
			\end{minipage}
		}
            \hspace{0.00pt}
		\subfigure[Instance $\mu^{\ddagger}$]{
			\begin{minipage}[b]{0.45\textwidth}  
				\includegraphics[width=1.000\textwidth]{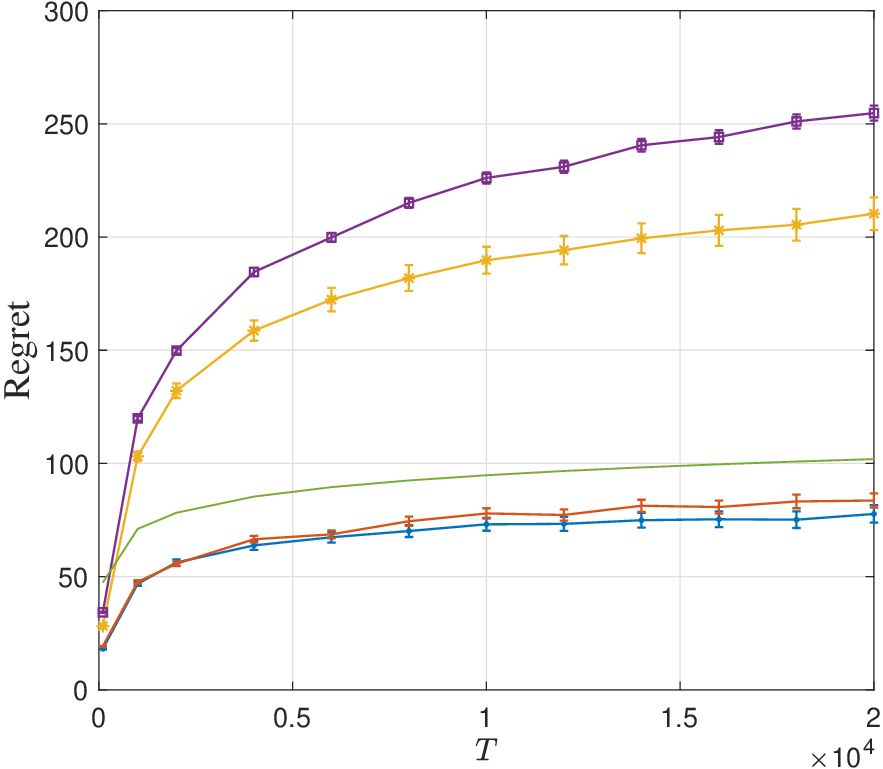}
			\end{minipage}
		}
	\end{minipage}
	\caption{Empirical regrets with abstention reward $c=0.9$ for different time horizons~$T$.}
	\label{fig_FRW_main1}
\end{figure}

\begin{figure}[t]
\centering
	\begin{minipage}{0.79\linewidth}
        \hspace{0.00pt}
		\subfigure[Instance $\mu^{\dagger}$]{
			\begin{minipage}[b]{0.45\textwidth}
				\includegraphics[width=1.000\textwidth]{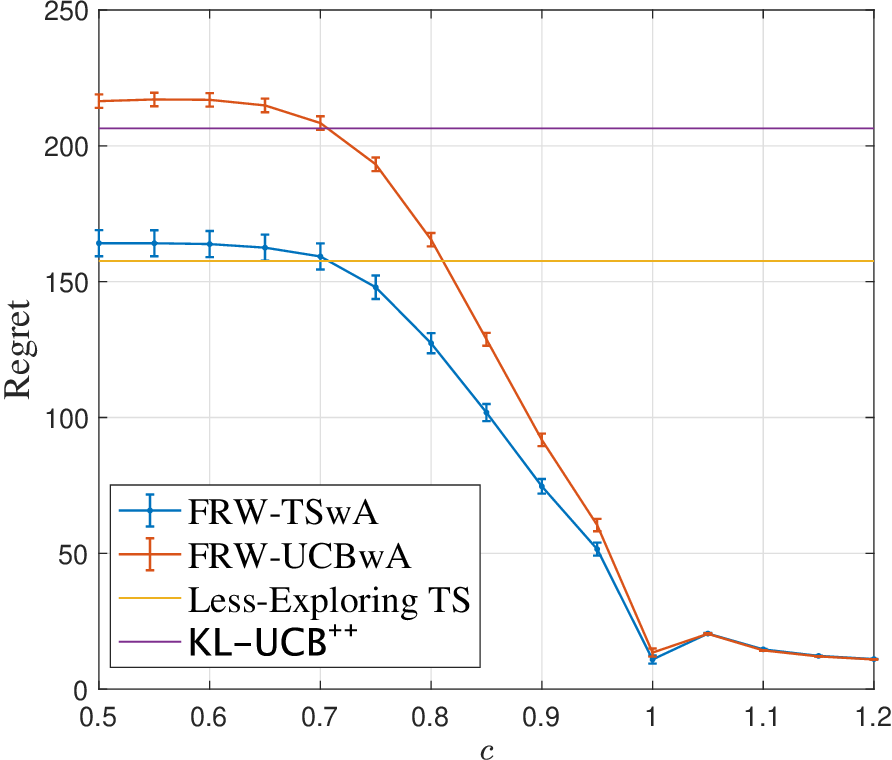}
			\end{minipage}
		}
  \hspace{0.00pt}
		\subfigure[Instance $\mu^{\ddagger}$]{
			\begin{minipage}[b]{0.45\textwidth}  
				\includegraphics[width=1.000\textwidth]{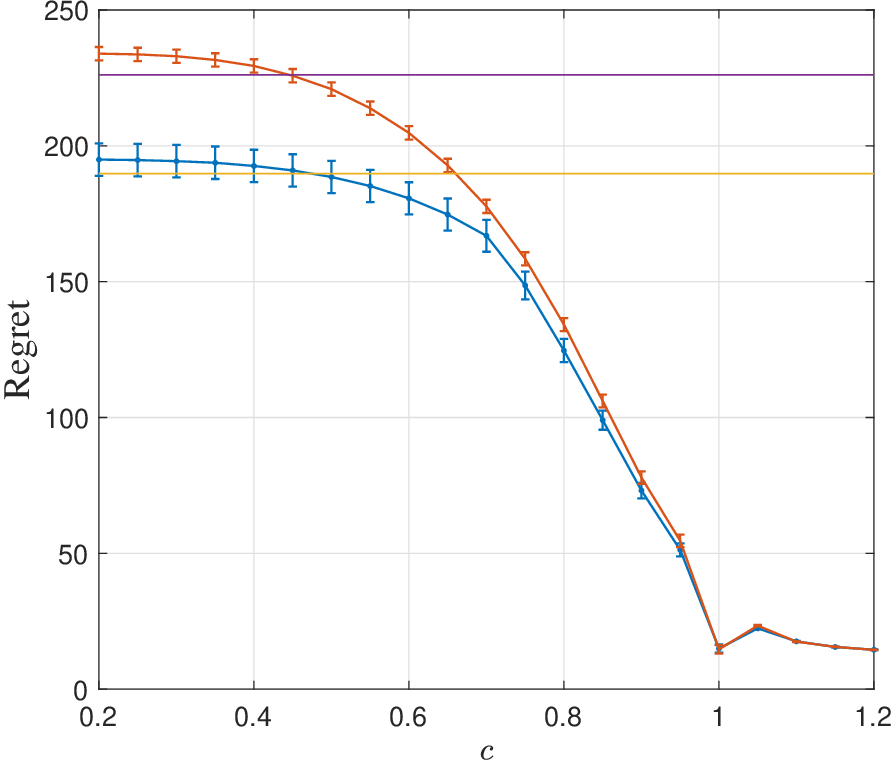}
			\end{minipage}
		}
	\end{minipage}
	\caption{Empirical regrets  with time horizon $T=10,000$ for different abstention rewards~$c$.}
	\label{fig_FRW_main2}
\end{figure}

In this subappendix, we present the empirical results pertaining to the fixed-reward setting. Specifically, we examine the empirical performances of two particular realizations of our algorithm \textsc{FRW-ALGwA} (as outlined in Algorithm~\ref{algo2}): \textsc{FRW-TSwA} and \textsc{FRW-UCBwA}. The former uses Less-Exploring Thompson Sampling \citep{jin23b} as its base algorithm, while the latter employs $\text{KL-UCB}^{++}$ \citep{menard2017minimax}. Note that $\text{KL-UCB}^{++}$ is not an anytime algorithm; that is, it requires the prior knowledge of the time horizon $T$ as an input parameter. In the following experiments, we continue to utilize the two bandit instances $\mu^{\dagger}$ and $\mu^{\ddagger}$, as previously defined in Appendix~\ref{subappendix_fixedreward_exp}.

In a manner analogous to our methodology for the fixed-regret setting, we adopt the original versions of Less-Exploring TS and $\text{KL-UCB}^{++}$ as baseline algorithms without the abstention option. The experimental results of the different methods with abstention reward $c=0.9$ for different time horizons $T$ are presented in Figure~\ref{fig_FRW_main1}. Additionally, we plot the instance-dependent asymptotic lower bound (ignoring the limit in $T$) on the cumulative regret (see Theorem~\ref{theorem_fixedreward_lowerbound}) within each sub-figure. From Figure~\ref{fig_FRW_main1}, we have the following observations:
\begin{itemize}
    \item Both realizations of our algorithm, \textsc{FRW-TSwA} and \textsc{FRW-UCBwA}, exhibit marked superiority over the two non-abstaining baselines.
    
    \item Concerning the observed growth trend, as the time horizon $T$ increases, the performance curves for both \textsc{FRW-TSwA} and \textsc{FRW-UCBwA} approximate the asymptotic lower bound closely. This behavior indicates that the expected cumulative regrets of \textsc{FRW-TSwA} and \textsc{FRW-UCBwA} attain the instance-dependent lower bound asymptotically, validating the theoretical findings discussed in Section~\ref{section_fixedreward}.
    
    \item While both \textsc{FRW-TSwA} and \textsc{FRW-UCBwA} represent implementations of our general algorithm and share identical theoretical guarantees, \textsc{FRW-TSwA} demonstrates superior empirical performance. This is particularly evident in the first instance $\mu^{\dagger}$, suggesting its enhanced applicability for real-world applications.
\end{itemize}

Next, we examine the impact of the abstention reward $c$ by assessing the performance of \textsc{FRW-TSwA} and \textsc{FRW-UCBwA} for different $c$,  while keeping the time horizon $T$ fixed at $10,000$.
The experimental results for bandit instances $\mu^{\dagger}$ and $\mu^{\ddagger}$ are shown in Figure~\ref{fig_FRW_main2}. 

Within each sub-figure, a pattern emerges. As the abstention reward $c$ increases, the empirical average cumulative regret initially remains relatively stable and starts to decline once $c$ crosses a certain threshold, eventually stabilizing around a small value. These observations are consistent with our theoretical expectations. Specifically, when the abstention reward $c$ is lower than the smallest mean reward among the arms, the agent derives no benefit from opting for the abstention action over selecting an arm. On the other hand, when the abstention reward $c$ exceeds the highest mean reward of the arms, abstention becomes the optimal decision and its reward is even superior to choosing the best arm. In this specific scenario, it is possible to achieve a regret of $o(\log T)$; see Remark~\ref{remark_c_large} for further insights.

\subsection{Random Instances}

\begin{figure}[t]
\centering
	\begin{minipage}{0.79\linewidth}
        \hspace{0.00pt}
		\subfigure[$K=20$]{
			\begin{minipage}[b]{0.45\textwidth}
				\includegraphics[width=1.000\textwidth]{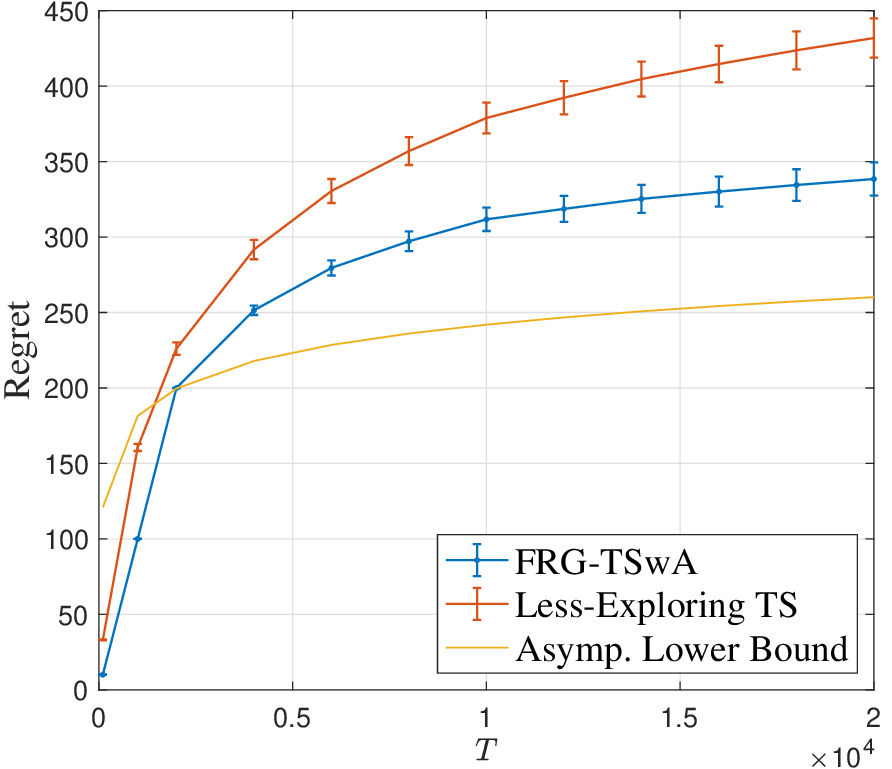}
			\end{minipage}
		}
  \hspace{0.00pt}
		\subfigure[$K=30$]{
			\begin{minipage}[b]{0.45\textwidth}  
				\includegraphics[width=1.000\textwidth]{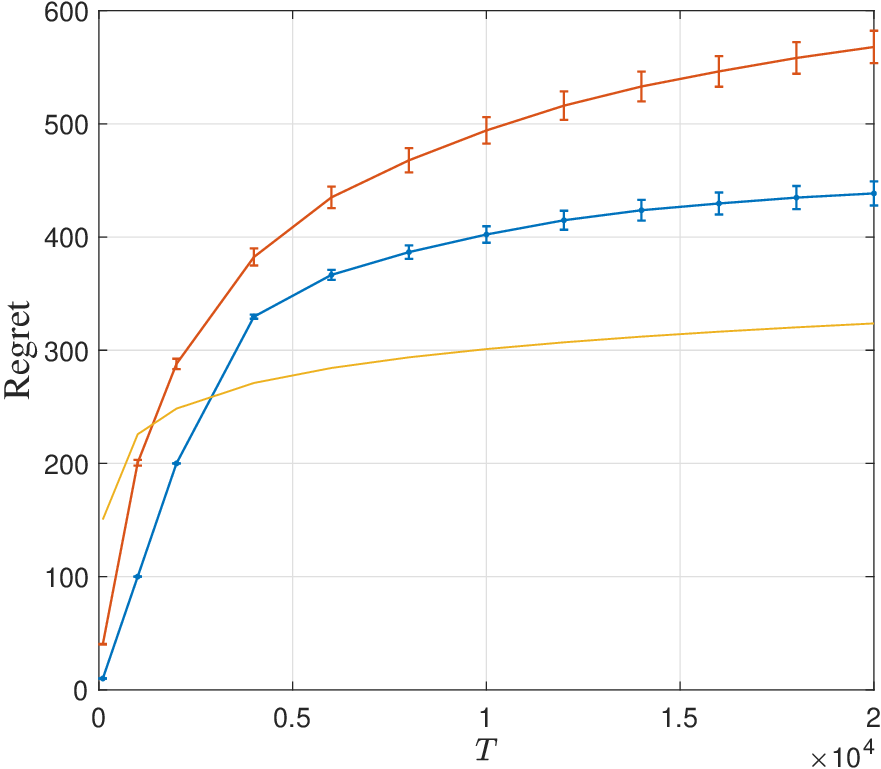}
			\end{minipage}
		}
	\end{minipage}
	\caption{Empirical regrets with abstention regret $c=0.1$ for different time horizons~$T$.}
	\label{fig_FRG_main3}
\end{figure}

\begin{figure}[t]
\centering
	\begin{minipage}{0.79\linewidth}
        \hspace{0.00pt}
		\subfigure[$K=20$]{
			\begin{minipage}[b]{0.45\textwidth}
				\includegraphics[width=1.000\textwidth]{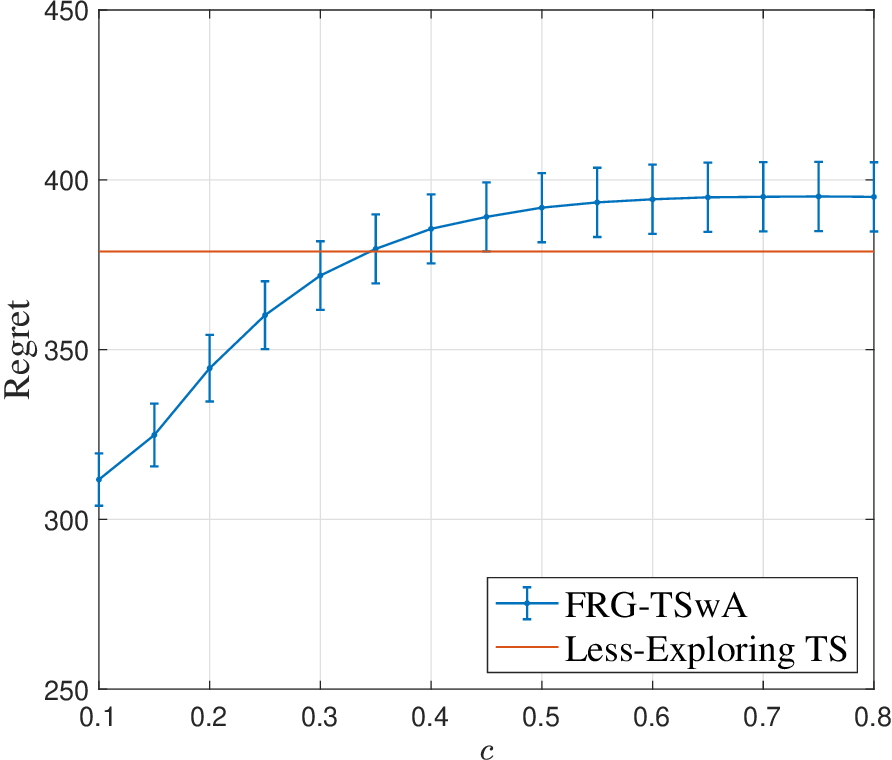}
			\end{minipage}
		}
  \hspace{0.00pt}
		\subfigure[$K=30$]{
			\begin{minipage}[b]{0.45\textwidth}  
				\includegraphics[width=1.000\textwidth]{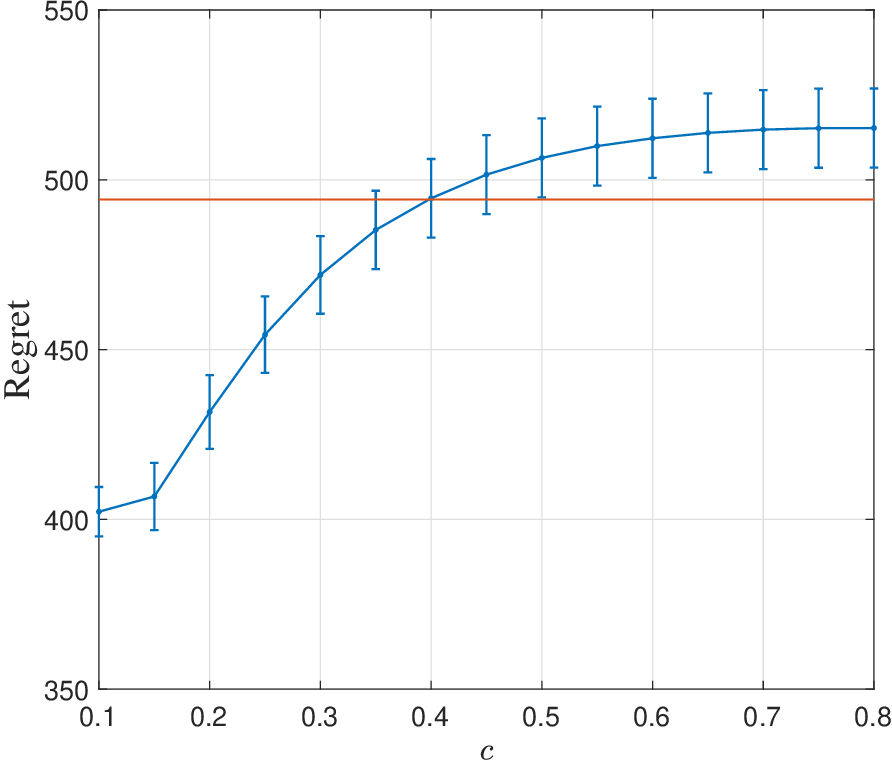}
			\end{minipage}
		}
	\end{minipage}
	\caption{Empirical regrets with time horizon $T=10,000$ for different abstention regrets~$c$.}
	\label{fig_FRG_main4}
\end{figure}

In this subappendix, we present additional numerical experiments, using random instances with large numbers of arms. The construction of these random instances mirrors  the method in \citet{jin23b}. Specifically, for a given number of arms denoted by $K\ge 10$, we set $\mu_1=1$ and $\mu_i=0.7$ for $i \in \{2,3, \ldots, 10\}$, while $\mu_i \sim \operatorname{Unif}[0.3,0.5]$ for $i \in [K] \setminus [10]$.

For the sake of simplicity in presentation, we focus on the fixed-regret setting, examining two choices of $K$, namely, $K=20$ and $K=30$. 
The empirical averaged cumulative regrets with abstention regret $c=0.1$ for different time horizons $T$ are shown in Figure~\ref{fig_FRG_main3}, while the experimental results with time horizon $T=10,000$ for different abstention rewards $c$ are illustrated in Figure~\ref{fig_FRG_main4}. 

It is evident that the findings in Figures~\ref{fig_FRG_main3} and \ref{fig_FRG_main4} closely resemble those in Figures~\ref{fig_FRG_main1} and \ref{fig_FRG_main2}. Notably, \textsc{FRW-TSwA} outperforms Less-Exploring TS which is not tailored to the setting with the abstention option.

\end{document}